\documentclass{article}

\usepackage{microtype}
\usepackage{graphicx}
\usepackage{subfigure}
\usepackage{booktabs} 
\usepackage{algolyx}

\usepackage{color}
\usepackage{dsfont}
\usepackage{lmodern}
\usepackage{float}
\usepackage{amssymb}
\usepackage{amsthm}
\usepackage[dvipsnames]{xcolor}
\usepackage{pifont}

\usepackage[english]{babel}

\newtheorem{theorem}{Theorem}
\newtheorem{lemma}[theorem]{Lemma}

\usepackage{multirow}
\usepackage{makecell}

\usepackage{tabularx}
\usepackage{mathtools}
\usepackage{bbm}
\usepackage{comment}
\usepackage{wrapfig}

\usepackage{hyperref}

\usepackage[preprint, nonatbib]{neurips_2021}

\global\long\def\bx{\mathbf{x}}%

\global\long\def\bX{\mathbf{X}}%

\global\long\def\by{\mathbf{y}}%

\global\long\def\bK{\boldsymbol{K}}%

\global\long\def\and{\cap}%

\global\long\def\idenmat{\mathbf{I}}%

\global\long\def\mutinfo{\mathbb{I}}%

\newcommand{\tick}{\textcolor{ForestGreen}{\ding{51}}}

\newcommand{\cross}{\textcolor{BrickRed}{\ding{55}}}

\title{Tuning Mixed Input Hyperparameters on the Fly for Efficient Population Based AutoRL}

\author{Jack Parker-Holder\\
	jackph@robots.ox.ac.uk \\
       University of Oxford
       \And
       Vu Nguyen\\
       vutngn@amazon.com \\
       Amazon
        \AND
       Shaan Desai \\
       shaan@robots.ox.ac.uk \\
       University of Oxford
        \And
       Stephen J. Roberts \\ 
       sjrob@robots.ox.ac.uk \\
       University of Oxford}

\begin{document}

\maketitle

\begin{abstract}
Despite a series of recent successes in reinforcement learning (RL), many RL algorithms remain sensitive to hyperparameters. As such, there has recently been interest in the field of AutoRL, which seeks to automate design decisions to create more general algorithms. Recent work suggests that population based approaches may be effective AutoRL algorithms, by learning hyperparameter schedules on the fly. In particular, the PB2 algorithm is able to achieve strong performance in RL tasks by formulating online hyperparameter optimization as time varying GP-bandit problem, while also providing theoretical guarantees. However, PB2 is only designed to work for \emph{continuous} hyperparameters, which severely limits its utility in practice. In this paper we introduce a new (provably) efficient hierarchical approach for optimizing \emph{both continuous and categorical} variables, using a new time-varying bandit algorithm specifically designed for the population based training regime. We evaluate our approach on the challenging Procgen benchmark, where we show that explicitly modelling dependence between data augmentation and other hyperparameters improves generalization.
\end{abstract}

\section{Introduction}
\label{sec:intro}

Reinforcement Learning (RL \cite{suttonbarto}) is a paradigm whereby agents learn to make sequential decisions through trial and error. In the past few years there have been a series of breakthroughs using RL in games \cite{alphago, dqn, dota} and robotics \cite{dexterous_openai, qt_opt}, which have led to a surge of interest in the machine learning community. Beyond these examples, RL offers the potential to impact vast swathes of society, from autonomous vehicles to robotic applications in healthcare and industry. 

Despite the promising results, RL is notoriously difficult to use in practice. In particular, RL algorithms are incredibly sensitive to hyperparameters \cite{deeprlmatters, andrychowicz2020matters, Engstrom2020Implementation, reproducibility_cont_control_islam_arxiv17, obando20revisiting}, with careful tuning often the difference between success and failure. Furthermore, as new and more complex algorithms are introduced, the search space continues to grow \cite{andrychowicz2020matters}. As a result, many lauded approaches are impossible to reproduce without the vast resources required to sweep through hundreds or thousands of possible configurations \cite{hpo_reproducible}. This makes it almost impossible to apply these methods in novel settings, where optimal hyperparameters are unknown. On the other hand, the capability of current methods may be understated, as better configurations could boost performance \cite{bo_alphago, pbtbt}.

In this work we focus on the recent impressive results for \emph{Population Based Training} (PBT, \cite{pbt, pbt2}), which has demonstrated strong performance in a variety of prominent RL settings \cite{kickstarting, pbtbt, liu2018emergent, impala, Jaderberg859}. PBT works by training agents in parallel, with an evolutionary outer loop optimizing hyperparameters, periodically replacing weaker agents with perturbations of stronger ones. However, since PBT relies on random search to explore the hyperparameter space, it requires a large population size which can be prohibitively expensive for small and even medium-sized labs. 

\begin{table*}[h]
\begin{center}
\vspace{-1mm}
\caption{\small{The components of related approaches, and their relative trade-offs.}}
\label{table:summary_of_diff}
\scalebox{0.9}{
\begin{tabular}{ cccc } 
\toprule
\textbf{Algorithm}  & \textbf{Population Based?} & \textbf{Efficient Continuous?} &  \textbf{Efficient Categorical?} \\ \midrule
PBT/PBA \cite{pbt, PBA}  & \tick  & \cross & \cross \\ 
PB2 \cite{pb2}  & \tick  & \tick & \cross \\ 
CoCaBO \cite{cocabo}  & \cross  & \tick & \tick \\ 
This work  & \tick  & \tick & \tick \\ 
\bottomrule
\end{tabular}}
\vspace{-1mm}
\end{center}
\end{table*}

In order to achieve similar success with a smaller computational budget, the recent \emph{Population Based Bandits} (PB2, \cite{pb2}) algorithm improved sample efficiency by introducing a probabilistic exploration step, backed by theoretical guarantees. Unfortunately, a key limitation of PB2 is that it only addresses the problem of efficiently selecting \emph{continuous} hyperparameters, inheriting the random search method for \emph{categorical} variables from the original PBT. This is not only inefficient but also crucially ignores potential dependence between continuous and categorical variables. Since the GP model is completely unaware of the changing categories, it is unable to differentiate between trials that may have completely different categorical hyperparameters. In this paper we introduce a new hierarchical approach to PB2, which can efficiently model \emph{both} continuous and categorical variables, with theoretical guarantees. Our main contributions are the following:

\textbf{Technical}: We introduce a new PB2 explore step which can \emph{efficiently} choose between \emph{both} categorical and continuous hyperparameters in a population based training setup. In particular, we propose a new time-varying batch multi-armed bandit algorithm, and introduce two hierarchical algorithms which condition on the selected categorical variables. We show our new approach achieves sublinear regret, extending the results for the continuous case with PB2. 

\textbf{Practical}: We scale our approach to test generalization on the Procgen benchmark \cite{procgen}, an active area of research. We demonstrate improved performance when explicitly modelling dependence between data augmentation type and continuous hyperparameters (such as learning rate) vs. baselines using random search to select the data augmentation.

\section{Related Work}
\label{sec:related} 

This work contributes to the emerging field of AutoRL \cite{autorl}, which seeks to automate elements of the reinforcement learning (RL) training procedure. Automating RL hyperparameter tuning has been studied since the 1980s \cite{barto1981goal}, with a surge in recent interest due to the increasing complexity of modern algorithms \cite{pbtbt, hoof}. The scope for AutoRL is broad, from using RL in novel real-world problems where hyperparameters are unknown \cite{runge2018learning}, to improving the performance of existing methods \cite{bo_alphago}. Recent successes in AutoRL include learning differentiable hyperparameters with meta-gradients  \cite{metagradients,stac}, while it has recently been shown it is even possible to learn algorithms \cite{learnedPG, evolving_algos}. 

In this paper we focus on the class of \emph{Population Based Training} (PBT \cite{pbt}) methods. Inspired by how humans run experiments, PBT trains agents in parallel and periodically replaces the weakest agents with variations of the stronger ones, learning a hyperparameter \emph{schedule} on the fly, in a single training run. While PBT is applicable in any AutoML \cite{automl_book} setting, it has been particularly impactful in deep RL \cite{impala,pbtbt, kickstarting}. PBT was recently shown to work well with a shared replay buffer \cite{franke2021sampleefficient}, but this only applies for off policy methods. This paper builds on the recently introduced \emph{Population Based Bandits} (PB2, \cite{pb2}) algorithm, which employs a Bayesian Optimization (BO, \cite{bayesopt_nando, bo_jmlr, practical_bo, gp_bucb, async_alvi, cocabo}) approach to select hyperparameters during the course of training. Concretely PB2 casts the hyperparameter selection step of PBT as a batch time-varying GP bandit optimization problem \cite{bogunovic2016time}. The strength of PB2 lies in finding optimal configurations with a small population size, yet at present it has only been developed for continuous hyperparameters, relying on random search for the categorical variables.

Categorical variables are prominent in RL, for example the choice of exploration strategy \cite{exploration_benchmark}, or even algorithm class. A prominent recent categorical variable is the choice of data augmentation, which has been shown to significantly improve efficiency and generalization in RL \cite{rad, drq, curl, mixreg}, reducing observational overfitting \cite{Song2020Observational}. Recent work introduced \emph{automatic data augmentation} \cite{ucb_drac}, which shows significant improvement over a static baseline by learning the data augmentation on the fly (with a single agent). However, this approach relies on a grid search over new hyperparameters, keeping all others fixed. We take inspiration from this result and learn \emph{both} data augmentation and baseline RL algorithm hyperparameters jointly on the fly, with a population of agents.  PBT has also been extended to data augmentation (PBA, \cite{PBA}) with strong results in supervised learning tasks. However, in this setting the data augmentation was set up as a continuous variable. A high level summary of differences vs. prior work is shown in Table \ref{table:summary_of_diff}.


\section{The Case for AutoRL}
\label{sec:background}

In this section we introduce the reinforcement learning (RL, \cite{suttonbarto}) paradigm, before making the case for automating hyperparameter selection for policy gradient algorithms.

\subsection{Reinforcement Learning Background}

A Markov Decision Process (MDP) is a tuple $(\mathcal{S}, \mathcal{A}, P, R, \gamma)$, where for each time step $t=0, 1, 2\ldots$ the environment provides the agent with an observation $s_t \in \mathcal{S}$, the agent responds by selecting an action $a_t \in \mathcal{S}$, and then the environment provides the next reward $r_{t}$, discount $\gamma_{t+1}$, and state $s_{t+1}$. Reinforcement learning (RL) considers the problem of learning a policy $\pi$ which selects actions that maximize the expected total discounted reward \cite{Bellman:DynamicProgramming,szepesvari2010algorithms,suttonbarto}. 

We note that our method can be applied to \textbf{any} RL algorithm.\footnote{In fact, PBT-style hyperparameter tuning can be applied to other machine learning paradigms such as supervised learning, but it has shown to be especially competitive in RL.} However, for simplicity we focus our discussion on \emph{policy gradient} algorithms, which have been widely studied in the community. Policy gradient algorithms directly seek to maximize a policy  $\pi:\mathcal{S}\rightarrow\mathcal{A}$ parameterized by $\theta$ with respect to expected reward, with an objective as follows:
\begin{align*}
    J(\pi_\theta) = \mathbb{E}_{\tau \sim \pi_\theta}R(\tau)
\end{align*}
where $\tau = \{s_1, a_1, r_1, \dots, s_H, a_H, r_H\}$ for some horizon $H$. Of course, we wish to maximize this, so optimize the policy by taking the following gradient steps:
\begin{align*}
    \theta_{t+1} = \theta_t + \alpha \nabla_\theta J(\pi_{\theta_t})
\end{align*}

\subsection{Case study: Proximal Policy Optimization}

Proximal Policy Optimization (PPO \cite{schulman2017proximal}) is one of the most widely used RL algorithms, achieving strong performance in continuous control problems, and even scaling to large scale games \cite{dota}. Building on Trust Region Policy Optimization (TRPO \cite{schulman2015trust}), the success of PPO comes from using a clipped loss function as follows:
\begin{align*}
     \mathcal{L}_{\mathrm{PPO}}(\theta) = \mathrm{min} \Bigg( \frac{\pi_\theta(a|s)}{\pi_\mu(a|s)} A^{\pi_\mu}, g(\theta, \mu) A^{\pi_\mu}  \Bigg), \; \text{where} \; 
     g(\theta, \mu) = \mathrm{clip} \bigg( \frac{\pi_\theta(a|s)}{\pi_{\mu}(a|s)}, 1-\epsilon, 1+\epsilon \bigg)
\end{align*}
for a previous policy $\pi_\mu$, an advantage function $A$ and a clipping hyperparameter $\epsilon$. As can be seen, PPO's success hinges on several new hyperparameters. As well as the learning rate $\alpha$, clip parameter $\epsilon$, decay $\gamma$, there is also often another hyperparameter in the advantage function $\lambda$ \cite{gae}. As if this wasn't enough, a recent study showed we also need to consider \emph{code level} hyperparameters \cite{Engstrom2020Implementation} after recent work found nine implementation details (such as reward scaling) had a significant impact on performance. Thus, it is clear that while PPO can be effective, it requires careful tuning if it is to be used effectively for novel problems.

In addition, a recent study investigating actor-critic algorithms \cite{andrychowicz2020matters} found that not only are there many more hyperparameters than previously thought (for example weight initialization strategy), but their optimal values are dependent on one another, producing a combinatorially large search space. This problem is only increasing as novel methods extending upon PPO they almost always add \emph{additional} hyperparameters \cite{ucb_drac}. Thus, if it wasn't enough that PPO itself has tens of hyperparameters, new additions may render all previous values sub-optimal since they may be dependent on the new variables. The main hypothesis in this paper is that we can improve performance of baseline algorithms by \emph{jointly} learning several hyperparameters \emph{on the fly}.

\section{Population Based Bandits}

We consider the problem of optimizing a population of agents in parallel, dynamically adapting their weights and hyperparameters such that strong performance can be achieved in a single training run. Following existing work \cite{pbt, pb2}, we consider two sub-routines, $\mathrm{explore}$ and $\mathrm{exploit}$. We train for a total of $T$ steps, evaluating performance every $t_{\mathrm{ready}} < T$ steps. For the $\mathrm{exploit}$ step, the weights of the bottom agents are replaced by those from a randomly sampled agent from the set of best performing agents, in what is called \emph{truncation selection}. The next step is to select new hyperparameters, with the  $\mathrm{explore}$ step. The full procedure is shown in Algorithm \ref{Alg:pb2} and Fig. \ref{figure:pbt}.

\scalebox{0.95}{
\begin{minipage}{.99\linewidth}
    \begin{algorithm}[H]
    \textbf{Initialize:} Population network weights, hyperparameters, \textcolor{blue}{empty dataset}. \\
    \textbf{for (in parallel)} $t=1, \ldots, T-1$ \textbf{do} \\
       1. \textbf{Train Models} \\
      2. \textbf{Evaluate Models \textcolor{blue}{and Record Data}}  \\
      3. If $t \mod t_{\mathrm{ready}} = 0$: 
      \begin{itemize}
          \item \textbf{Exploit:} Replace the weights of weaker agents with those from stronger agents.
          \item \textbf{Explore:} If weights were replaced, select new hyperparameters.
      \end{itemize}
     \textbf{Return the best trained model $\theta$}
     \caption{Population Based Training (\textcolor{blue}{Bandits})}
    \label{Alg:pb2}
    \end{algorithm}
\end{minipage}
}

The key difference between Population Based Training (PBT, \cite{pbt}) and Population Based Bandits (PB2, \cite{pb2}) is the $\mathrm{explore}$ step. PBT selects new hyperameters using random perturbations, hence it often performs poorly in a resource constrained setting. By contrast PB2 (\textcolor{blue}{blue} in Alg. \ref{Alg:pb2}) models the data from previous trials to efficiently explore the hyperparameter space. We focus on the PB2 framework and introduce new algorithms for mixed-input hyperparameters.

\subsection{Online Hyperparameter Selection as GP-Bandit Optimization}\label{subsec:GPbanditopt}

\begin{wrapfigure}{R}{0.50\textwidth}
    \centering
    \vspace{-7mm}
    \centering\subfigure{\includegraphics[width=.9\linewidth]{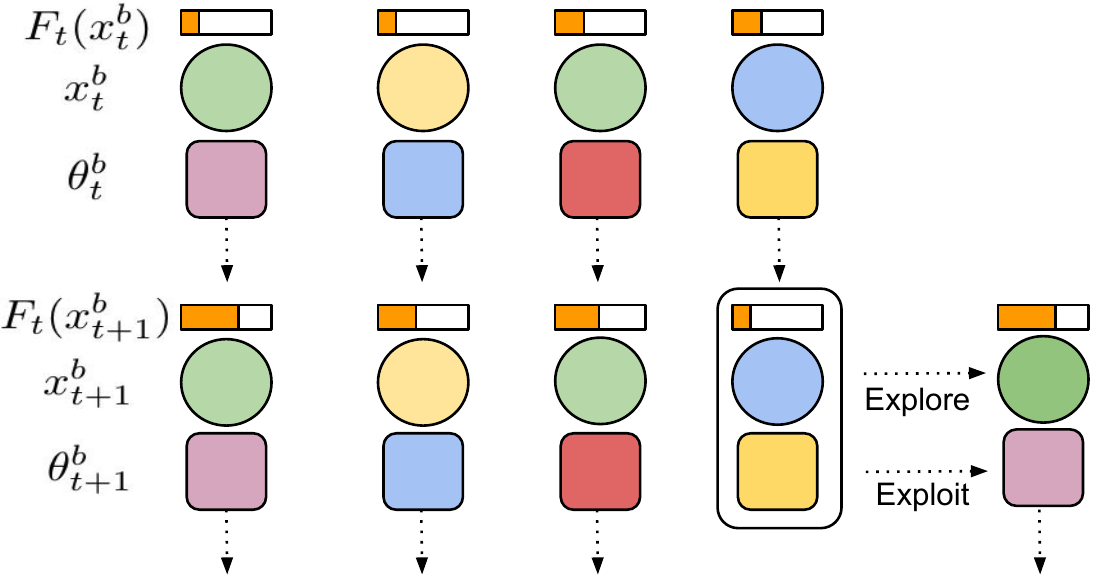}}
    \vspace{-4mm}
    \caption{\small{An overview of the PBT/PB2 framework. We have a population of $B$ agents, each with weights $\theta_t^b$ and hyperparameters $x_t^b$. $t$ refers to the stage of training, with each increment corresponding to training for $t_\mathrm{ready}$ steps. The performance of the agent is represented by $F_t$. At each $t$, the weights of the worst performing agents are replaced ($\mathrm{exploit}$), where weights are copied from better agents. The new hyperparameters are then selected using an $\mathrm{explore}$ procedure, the focus of this work.}}
    \vspace{-5mm}
    \label{figure:pbt}
\end{wrapfigure}

We consider the problem of selecting optimal hyperparameters, $x_t^b$, from a compact, convex subset $\mathcal{D}\in \mathbb{R}^d$ where $d$ is the number of hyperparameters. Here the index $b$ refers to the $b$th agent in a population, and the subscript $t$ represents the number of timesteps elapsed during the training of a neural network. In particular, we consider the \textit{schedule} of optimal hyperparameters over time $\left( x_t^b \right)_{t=1,...T}$. Let $F_t(x_t)$ be an objective function under a given set of hyperparameters at timestep $t$. Here we take $F_t(x_t)$ to be the reward for a deep RL agent. Our goal is to maximize the final performance $F_T(x_T)$. 

Following \cite{pb2} we formulate this problem as optimizing the time-varying black-box reward function $f_t$, over $\mathcal{D}$. Every $t_{\mathrm{ready}}$ steps, we observe and record noisy observations, $y_t = f_t(x_t) + \epsilon_t$, where $\epsilon_t \sim \mathcal{N}(0, \sigma^2\mathbf{I})$ for some fixed $\sigma^2$. Note the function $f_t$ represents the \emph{change} in $F_t$ after training for $t_{\mathrm{ready}}$ steps, i.e. $F_t - F_{t-t_\mathrm{ready}}$. We define the best choice at each timestep as $x_t^* = \arg\max_{x_t \in \mathcal{D}} f_t(x_t)$, and so the \textbf{regret} of each decision as $r_t = f_t(x_t^*) - f_t(x_t)$. Minimizing the regret of each decision is equivalent to maximizing the final reward \cite{pb2}, i.e. $\max F_T(x_T)= \min \sum^T_{t=1} r_t(x_t)$. We now discuss how PB2 seeks to minimize regret for \emph{continuous} hyperparameters.

\subsection{Parallel Gaussian Process Bandits for a Time-Varying Function}\label{sec:TVGP}

PB2 models the data using a Gaussian Process (GP, \cite{RasmussenGP}), thus assumes the data follows a Gaussian distribution with mean $\mu_t(x')$ and variance $\sigma_t^2(x')$ as: 
\begin{equation}
\label{eqn:mu}
    \mu_t(x') \coloneqq \textbf{k}_t(x')^T(\textbf{K}_t + \sigma^2\textbf{I})^{-1}\mathbf{y}_t
\end{equation}
\begin{equation}
\label{eqn:sig}
    \sigma_t^2(x') \coloneqq k(x', x') - \textbf{k}_t(x')^T(\textbf{K}_t +\sigma^2\textbf{I})^{-1}\textbf{k}_t(x'),
\end{equation}
where $\mathbf{K}_t \coloneqq \{k(x_i, x_j)\}_{i, j=1}^t$ and $\mathbf{k}_t \coloneqq \{k(x_i, x'_t)\}_{i=1}^t$. The GP predictive mean and variance above will later be used to represent the exploration-exploitation trade-off in making decisions in the presence of uncertainty. The key insight in PB2 is that it is possible to consider the time-varying nature of neural network hyperparameters. This is consistent with empirical results, for example, the use of learning rate schedules in almost all large scale neural network models \cite{alexnet}. PB2 follows \cite{bogunovic2016time} by modeling the reward function under the time-varying setting as follows:
\begin{equation*}
     f_1(x) = g_1(x), \; \; \; \;
     f_{t+1}(x) = \sqrt{1-\omega}f_t(x) + \sqrt{\omega}g_{t+1}(x) \; \: \: \:  \forall t \geq 2,
\end{equation*}
where $g_1, g_2, ...$ are independent random functions with $g\sim GP(0, k)$ and $\omega \in[0,1]$ models how the function varies with time, such that if $\omega=0$ we return to GP-UCB and if $\omega=1$ then each evaluation is completely independent. This leads to the extensions of Eqs. (\ref{eqn:mu}) and (\ref{eqn:sig}) 
using the new covariance matrix $\Tilde{\textbf{K}}_t = \textbf{K}_t \circ \textbf{K}_t^{\mathrm{time}}$ where $\textbf{K}_t^{\mathrm{time}} = [(1-\omega)^{|i-j|/2}]_{i,j=1}^T$ and $\Tilde{\textbf{k}}_t(x) = \textbf{k}_t \circ \textbf{k}_t^{time}$ with $\textbf{k}_t^{time} = [(1-\omega)^{(T+1-i)/2}]_{i=1}^T$. Here $\circ$ refers to the Hadamard product. 

\textbf{Selecting hyperparameters for parallel agents.} A key observation in \cite{gp_bucb} is that since a GP's variance (Eqn. \ref{eqn:sig}) does not depend on $y_t$, the acquisition function can  account for incomplete trials by updating the uncertainty at the pending evaluation points. Recall $x^b_{t}$ is the $b$-th point selected in a batch, after $t$ timesteps. This point may draw on information from $t + (b-1)$ previously selected points. In the single agent, sequential case, we set $B=1$ and recover $t,b = t-1$. Thus, at the iteration $t$, we find a next batch of $B$ samples $\left[x^1_t,x^2_t,...x^B_t \right]$  by sequentially maximizing the following acquisition function:
\begin{equation}
\label{eq:gp_bucb_acq}
    x^b_t = \arg \max_{x\in \mathcal{D}} \mu_{{t,1}}(x) + \sqrt{\beta_t} \sigma_{{t,b}}(x),\forall b=1,...B
\end{equation}
for $\beta_t >0$. In Eqn. (\ref{eq:gp_bucb_acq}) we have the mean from the previous batch ($\mu_{t,1}(x)$) which is fixed, but can update the uncertainty using our knowledge of the agents currently training ($\sigma_{t,b}(x)$). This significantly reduces redundancy, as the model is able to explore distinct regions of the space. 

To summarize, the PB2 explore step works by optimizing Eqn. (\ref{eq:gp_bucb_acq}) to select a batch of new parameters. However, it is clear to see that this approach is not equipped with a means to select \emph{categorical} variables. In fact, the existing PB2 algorithm inherits the random selection of categories from PBT. To reflect this, from this point onwards we refer to PB2 from \cite{pb2} as \textbf{PB2-Rand}. 

\section{Efficient Selection of Continuous and Categorical Variables}
\label{sec:method}
We propose a new class of hierarchical population based approaches for handling continuous and categorical variables. To the best of our knowledge, this is the first provably efficient approach for optimizing the categorical variables in the PBT family \cite{pbt,pb2}. First we introduce a new time-varying version of the parallel EXP3 algorithm, which we call $\mathrm{TV.EXP3.M}$, used in all of our methods to select categories. We then present three alternative approaches to subsequently select the continuous variables in a hierarchical fashion, varying in their degree of dependence on the categorical choices.

\subsection{Time-varying Parallel EXP3}

\begin{algorithm}
\caption{TV.EXP3.M brief\label{alg:EXP3_MultiplePlay_S_short}}
\label{alg:tv_exp3_brief}

\begin{algor}
\item [{{*}}] Input: $\gamma=\sqrt{\frac{C\ln(C/B)}{(e-1)BT}}$, $\alpha=\frac{1}{T}$, $C$ \#categorical
choice, $T$ \#max iteration, $B$ \#multiple play
\end{algor}
\begin{algor}[1]
\item [{{*}}] Init $w_{c}=1,\forall c=1...C$ and denote $\eta=(\frac{1}{B}-\frac{\gamma}{C})\frac{1}{1-\gamma}$
\item [{for}] $t=1$ to $T$
\item [{{*}}] Normalize $w$ to prevent from over exploitation
\item [{{*}}] Compute the prob for each arm $p_t^c, \forall c$ 
\item [{{*}}] Select a batch $S_{t}=\textrm{DepRound}\left(B,\left[p_{t}^{1}p_{t}^{2}....p_{t}^{C}\right]\right)$
\item [{{*}}] Observe the reward after evaluating $S_t$
\item [{{*}}] Update the weights using $\eta$ for each arm $w_t^c, \forall c$ 
\item [{endfor}]~
\end{algor}
\begin{algor}
\item [{{*}}] Output: $\mathcal{D}_{T}$
\end{algor}
\end{algorithm} 

We introduce TV.EXP3.M, a new algorithm for parallel multi-armed bandits (MAB) in a  time-varying setting with adversarial feedback. TV.EXP3.M is an extension of the multiple play EXP3 algorithm (or EXP3.M) \cite{uchiya2010algorithms} for time-varying rewards. At each round, TV.EXP3.M takes multiple actions from the set of all possible choices \cite{uchiya2010algorithms}, making it possible to use this approach with a population of agents. See Algorithm \ref{alg:EXP3_MultiplePlay_S_short} for a high level summary and Algorithm \ref{alg:EXP3_MultiplePlay_S} and Sec. \ref{appendix_subsec:tv_exp3_m} for further details.

We treat each categorical choice as an arm in the bandit setting. Given a collection of $C$ categories, each includes a reward probability, we select a batch of $B$ points
\begin{equation}
\label{eq:batch_selection_TVEXP3.M}
    A_t = [A_t^1, A_t^2,...A_t^B] = \mathrm{TV.EXP3.M}(D_{t-1}).
\end{equation}
To efficiently select a set of $B$ distinct arms from $[C]$, we use the technique of dependent rounding (DepRound) while satisfying the condition that each arm $c$ is selected with probability $p_c$ exactly
\cite{gandhi2006dependent}. TV.EXP3.M is provably efficient, with a sublinear regret bound.

\begin{theorem}
\label{thm_regret_TVEXP3M}
Set $\alpha=\frac{1}{T}$ and $\gamma=\min\left\{ 1,\sqrt{\frac{C\ln(C/B)}{(e-1)BT}}\right\}$, we assume the reward distributions changes at arbitrary instances, but the total number of change points is no more than $V \ll \sqrt{T}$ times. The expected regret of TV.EXP3.M  satisfies the following sublinear bound
\begin{align*}
\mathbb{E} \left[ R_{TB} \right]\le & \left[1+e+V\right] \sqrt{(e-1)\frac{CT}{B}\ln\frac{CT}{B}}.
\end{align*}
\vspace{-6mm}

\end{theorem}


\subsection{New Exploration Strategies}
Now we are ready to present two novel hierarchical exploration strategies. Both methods use TV.EXP3.M to select categories before subsequently choosing continuous variables. 

\begin{figure*}[h]
\begin{minipage}{0.56\textwidth}
\begin{algorithm}[H]
\small
\centering
\caption{PB2-Mult: explore}
\label{pb2mult}
\begin{algorithmic}
  \STATE{Select $\{h^b_t\}^B_{b=1}$ using TV.EXP3.M}
  \STATE{Filter dataset}
  \STATE{Select $x_t^b$ by optimizing Eqn. (\ref{eq:gp_bucb_acq})}
  \STATE{\textbf{Output:} $\{h^b_t\}^B_{b=1}$, $\{x^b_t\}^B_{b=1}$}.
\end{algorithmic}
\end{algorithm}
\end{minipage}
\hfill
\begin{minipage}{0.4\textwidth}
\begin{algorithm}[H]
\centering
\small
\caption{PB2-Mix: explore}
\label{pb2mix}
\begin{algorithmic}
\STATE Select $\{h^b_t\}^B_{b=1}$ using TV.EXP3.M
\STATE Select $x_t^b$ by optimizing Eqn. (\ref{eq:gp_bucb_acq}), using the kernel from Eqn. (\ref{eq:TV_Cocabo_kernel})
\STATE  \textbf{Output:} $\{h^b_t\}^B_{b=1}$, $\{x^b_t\}^B_{b=1}$
\end{algorithmic}
\end{algorithm}
\end{minipage}
\vspace{-1pt}
\end{figure*}

\textbf{PB2-Mult (Algorithm \ref{pb2mult}):} We consider a special version of dependency in which the presence of the continuous variables entirely depends on if the corresponding category is switched on. This not only allows us to fully incorporate dependence between hyperparameters, but also allows us to have different hyperparameters depending on the category. For example, if our categorical variable is the choice of neural network optimizer, then the current optimal learning rate may vary significantly if using Adam \cite{adam} or SGD. Furthermore, Adam requires additional hyperparameters, $\beta_1$ and $\beta_2$. We handle this dependency by constructing \emph{multiple} time-varying GP surrogate models, each corresponding to a choice of category. We call this method \textbf{PB2-Mult}. PB2-Mult is the best choice if the continuous variables are highly dependent on the category, or if there are differing numbers of continuous variables for each category. 

\textbf{PB2-Mix (Algorithm \ref{pb2mix}):} We also propose a mechanism to incorporate dependence between continuous and categorical variables \emph{directly into the GP kernel}, in what we call \textbf{PB2-Mix}. To do this, we extend the joint continuous-categorical kernel in CoCaBO \cite{cocabo} to the time-varying setting. Denote $z=[x,h]$, we have:
\vspace{-3mm}
\begin{align} \label{eq:TV_Cocabo_kernel}
k_{z}(z,z') & =(1-\lambda)\left( \Tilde{\textbf{k}}_t(x)+k_{ht}\right)+\lambda \Tilde{\textbf{k}}_t(x) k_{ht}
\end{align}
where $k_{ht}=\frac{\sigma_{2}}{C}\sum\mutinfo(h,h') \times \textbf{k}_t^{time}$ and 
$\Tilde{\textbf{k}}_t(x)$, $\textbf{k}_t^{time}$ are defined in Sec. \ref{sec:TVGP}. In the above formulation, we first select a batch of $B$ categorical choices $h^b_t,  \forall b\le B$ using TV.EXP3.M. Then, we utilize the kernel in Eqn. (\ref{eq:TV_Cocabo_kernel}) to learn the time-varying GP model to select the next batch of continuous points $x_t^b,\forall b\le B$ following Eqn. (\ref{eq:gp_bucb_acq}), as in PB2-Rand. Each agent $b$ will then train with $[x^b_t, h^b_t]$. We do not need to set the parameter $\lambda$ as it can be optimized alongside the other kernel parameters (see the Appendix, Sec: \ref{sec:kernel_gradients}). Thus, PB2-Mix in principle achieves the best of both worlds: it is able to learn the relative importance of the continuous and categorical kernels, while including all of the data. 

\subsection{Main theoretical results}

We extend the sublinear regret bound from PB2-Rand to PB2-Mult, providing performance guarantees. Under the categorical-continuous setting, we show the equivalence that minimizing the cumulative regret $R_{T}$  is equivalent to maximizing the reward, presented in Lem. \ref{lem:maxreward_minregret} in the appendix. Our main theoretical result is to derive the regret bound for PB2-Mult.  

\begin{theorem} \label{thm:PB2-D}
The PB2-Mult algorithm has cumulative regret bounded with high probability
\vspace{-1mm}
\begin{small}
\begin{align*}
\mathbb{E}\left[R_{TB} \right] & \lesssim \mathcal{O} \left(V \sqrt{\frac{CT}{B} \ln T } + \sqrt{ \frac{T^2_{c^*_t}}{\tilde{N}B}  ( \gamma_{\tilde{N}B} +\left[\tilde{N}B\right]^{3}\omega )} \right)  
\end{align*}
\end{small}
\vspace{-1mm}
for $\tilde{N} \rightarrow  T_{c^*_t}$, $\omega \rightarrow 0$, and $V \ll \sqrt{T}$.
\end{theorem}

From Theorem \ref{thm_regret_TVEXP3M}, we have $T_{c^*_t} \rightarrow \infty$ when $T \rightarrow \infty$. Asymptotically PB2-Mult achieves sublinear convergence rate, i.e., $\lim_{T\rightarrow \infty} \frac{\mathbb{E}\left[  R_{TB} \right]}{T}=0$. We refer to Sec. \ref{appendix_subsec:Theory_PB2_D} in the Appendix for proofs.


\section{Experiments}
\label{sec:experiments}

The primary goal of our experiments is to test the efficacy of PB2-Mult and PB2-Mix against PB2-Rand, the primary baseline. We consider two settings: first, a synthetic task where we know that the continuous and categorical variables exhibit strong dependence. We then move to the AutoRL setting, where we apply them to select both continuous hyperparameters and categorical data augmentation type for challenging pixel-based procedurally generated environments \cite{ucb_drac}.

\subsection{Synthetic Function with Dependency}

We begin with a simple synthetic function with one continuous and one categorical variable. The categorical variable has two options $h_t^b \sim \{\mathrm{sin}, \mathrm{cos}\}$ and the continuous variable has bounds $x_t^b \sim [0, \frac{\pi}{2}]$. The blackbox function $f: [x,h]\rightarrow \mathbb{R}$ returns $h(x)$, and the regret is equal to $1-f_t(.)$. It is clear to see the function is maximized (and regret minimized) at $\{\mathrm{sin}, \frac{\pi}{2}\}$ and $\{\mathrm{cos}, 0\}$. We compare PB2-Mult and PB2-Mix against PB2-Rand, PBT and Random Search. At each iteration, the worst performing agent(s) is replaced by the best and the categories and continuous variables are re-selected.  We run the example $20$ times, with a population size of $4$, $8$ or $12$ agents. 

\begin{wrapfigure}{R}{0.55\textwidth}
    \centering
    \vspace{-7mm}
    \centering\subfigure{\includegraphics[width=.99\linewidth]{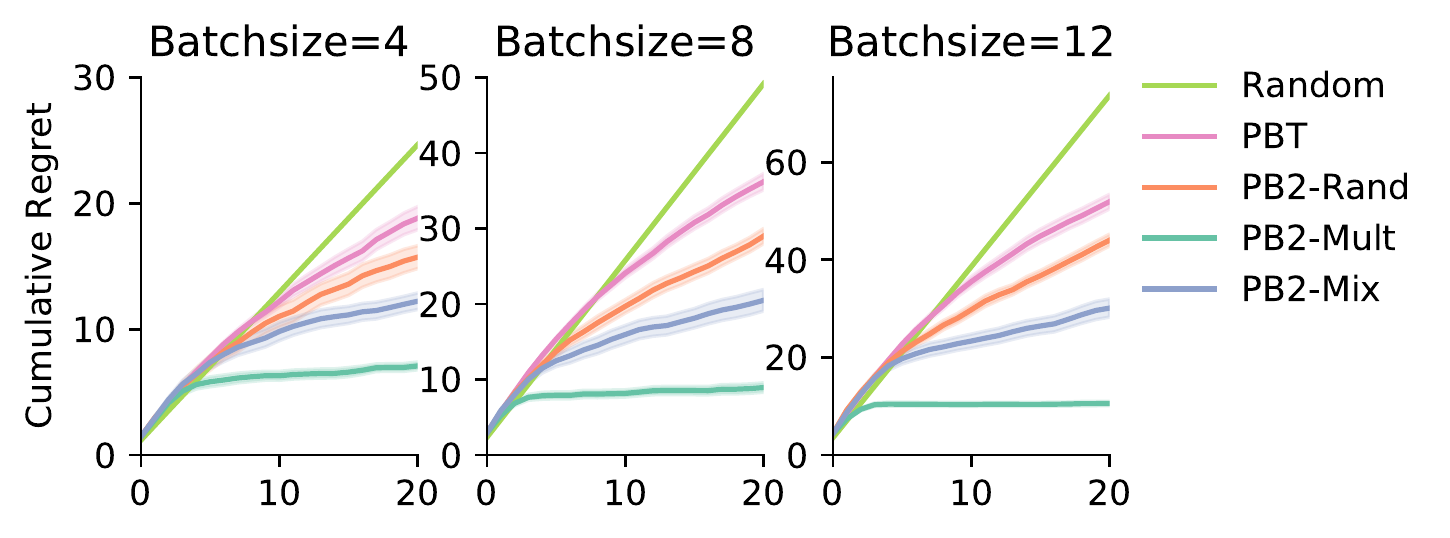}} 
    \vspace{-7mm}
    \caption{\small{Mean $\pm1$sem cumulative regret for 20 runs.}}
    \vspace{-4mm}
    \label{figure:toy prob}
\end{wrapfigure}

Fig \ref{figure:toy prob} shows the mean cumulative regret, with the standard error (sem) shaded. As we see, random search has linear regret, while PBT and PB2-R both provide incremental improvements. Our new approaches outperform, led by PB2-Mult. Thus, we see that when there is strong dependence between the two hyperparameters, the methods which explicitly model this significantly outperform. For details, see the following notebook \url{http://bit.ly/synthetic_func}, including open source implementations of all the discussed algorithms.

\subsection{Learning Hyperparameters and Data Augmentation for Generalization in RL}

There has recently been increased interest in testing RL algorithms procedurally generated environments \cite{sonic, procgen}. These typically consist of a set of parameters which are used to generate \emph{levels}, consisting of changes to the observation or layout of the environment, requiring agents to generalize \cite{procgen1, procgen2}. In this paper we focus on the Procgen environment \cite{procgen}. Given the recent success of data augmentation in Procgen, we adopt the author's implementation of DrAC \cite{ucb_drac}. As well as the categorical data augmentation, we also tune the regularization coefficient $\alpha_r$, the PPO clip parameter $\epsilon$, the learning rate and entropy exploration coefficient. The bounds for each, as well as the fixed hyperparameters, are given in the Appendix (Section \ref{sec:implementation_details}, Table \ref{table:drac_fixed} and \ref{table:drac_learned}). We learn these hyperparameters, as well as the categorical data augmentation \emph{on the fly} in a single training run, starting with random samples from a wide range.

\begin{figure}[H]
    \centering
    \vspace{-2mm}
    \begin{minipage}{0.99\textwidth}
    \subfigure{\includegraphics[width=.99\linewidth]{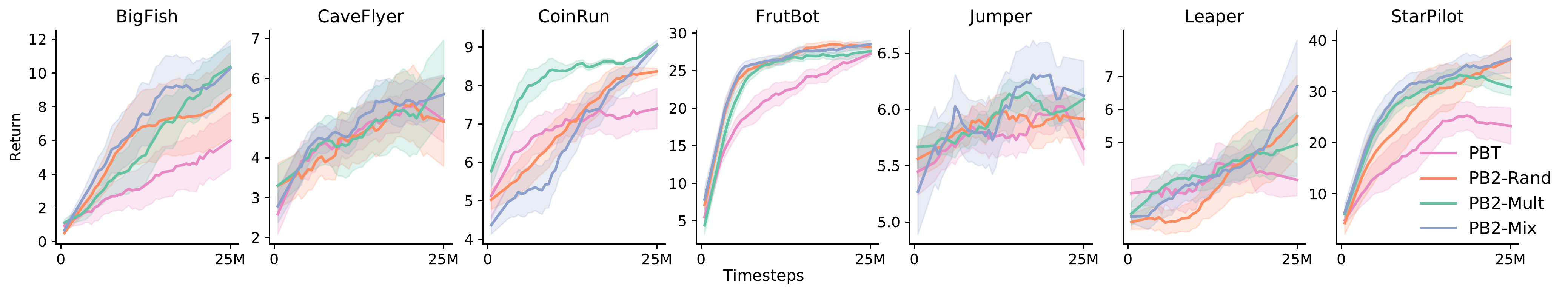}} 
    \vspace{-4mm}
    \caption{\small{\textbf{Test Performance}: Learning curves for seven Procgen games. Plots show the mean $\pm1$sem for the test performance of the best training agents in the population. Results are averaged over $5$ seeds.}}
    \label{figure:procgen_test}
    \vspace{-3mm}
    \end{minipage}
\end{figure}

We train a population of $B=4$ agents, each for $25$M steps, with $t_{\mathrm{ready}} = 500$k steps, thus the explore step is called approximately fifty times. We use seven Procgen games: BigFish, CaveFlyer, CoinRun, FruitBot, Jumper, Leaper and StarPilot. At the end, each agent is evaluated on $100$ train and test levels and we present the test performance from the agent with the best final training performance in each population. Our primary baseline is PB2-Rand, which we note was already shown to improve upon both Hyperband \cite{hyperband, asha} and vanilla BO \cite{bayesopt_nando} in \cite{pb2}. As such, we choose to use our computational resources to run five trials for all seven games rather than by confirming this result. Instead, we do include the original PBT algorithm as a baseline, as well as single agent results from \cite{ucb_drac}.  

\begin{table}[H]
\vspace{-2mm}
\begin{center}
\caption{\small{Test performance for seven Procgen games. $\dagger$ indicates training was conducted for a single agent for 25M timesteps. Hyperparameters were initialized at optimized values, meaning the effective number of timesteps is much higher. Final performance is the average of 100 trials and results were taken from \cite{ucb_drac}. $\ddagger$ indicates training was conducted by a population of four agents for 25M timesteps each, with several hyperparameters initialized at random. Each agent is evaluated for $100$ trials on train and test levels and we present the mean test performance of the agent with the best training performance. The ``Normalized Returns'' are with respect to the PB2-Rand baseline, $\star$ indicates $p<0.05$ in a t-test over PB2-Rand. Bold = within $1$ std of the best mean.
}}
\vspace{-3mm}
\label{table:procgen_results_new}
\scalebox{0.82}{
\begin{tabular}{ l | cc | cc | cc } 
\toprule
\textbf{Environment} & PPO$\dagger$  & UCB-DrAC$\dagger$ & PBT$\ddagger$ & PB2-Rand$\ddagger$ & PB2-Mult$\ddagger$ & PB2-Mix$\ddagger$  \\ 
\midrule 
BigFish & $4.0 \pm 1.2$ & $\mathbf{9.7 \pm 1.0}$ & $6.6\pm2.4$ & $8.2\pm0.1$ & $\mathbf{9.6\pm1.7}$ & $\mathbf{10.6\pm1.9}$ \\ 
CaveFlyer & $5.1 \pm 0.9$ & $\mathbf{5.3 \pm 0.9}$ & $4.4\pm1.0$ & $\mathbf{5.8\pm1.8}$ & $\mathbf{5.6\pm0.8}$ & $\mathbf{6.0\pm0.7}$ \\ 
CoinRun & $\mathbf{8.5\pm0.5}$ & $\mathbf{8.5\pm0.5}$ & $6.8\pm1.0$ & $\mathbf{8.1\pm0.1}$ & $\mathbf{8.5\pm0.4}$ & $\mathbf{8.1\pm0.2}$ \\ 
FruitBot & $26.7 \pm0.8$ & $28.3 \pm 0.9$ & $25.6\pm0.2$ & $\mathbf{29.5\pm1.9}$ & $\mathbf{29.6\pm0.8}$ & $\mathbf{29.0\pm0.6}$ \\ 
Jumper & $5.8\pm0.5$ & $\mathbf{6.4 \pm 0.6}$ & $\mathbf{6.2\pm0.3}$ & $5.4\pm0.0.5$ & $\mathbf{5.9\pm0.1}$ & $\mathbf{6.2\pm0.3}$ \\ 
Leaper & $\mathbf{4.9 \pm 0.7}$ & $\mathbf{5.0 \pm 0.3}$ & $3.6\pm0.6$ & $\mathbf{5.0\pm2.4}$ & $\mathbf{5.0\pm1.4}$ & $\mathbf{6.6\pm1.9}$ \\ 
StarPilot & $24.7 \pm 3.4$  & $30.2 \pm 2.8$ & $26.9\pm7.9$ & $33.6\pm3.1$ & $\mathbf{35.3\pm1.5}$ & $\mathbf{36.3\pm2.4}$ \\
\midrule
Normalized Returns (\%) &  &  & $85.6\pm22.4$ & $100\pm23.8$ & $105.3\pm14.6$ & $112.1\pm20.4^\star$ \\
\bottomrule
\end{tabular}}
\end{center}
\vspace{-5mm}
\end{table}

The final results are shown in Table \ref{table:procgen_results_new}, where each agent is evaluated for 100 trials, a common protocol in Procgen \cite{ucb_drac,jiang2020prioritized}. For population based methods, we select the agent with the best \emph{training performance} and show its' test performance. Given the challenge of comparing results, we present an additional metric which is a score normalized by the PB2-Rand performance. For each seed in each game, the score is divided by the mean PB2-Rand score for that given game. We then show the mean of these normalized scores, comprising of $35$ results. We also show full test learning curves in Fig. \ref{figure:procgen_test}, while training performance is shown in Fig. \ref{figure:procgen_train} in the Appendix.

Despite using a small population size, we are able to train policies that match or outperform state-of-the-art methods which used a much larger grid search. We see that PB2-Rand significantly outperforms PBT, and furthermore our new approaches modeling categorical variables improve upon PB2-Rand. Notably, we see the strongest performance from PB2-Mix, which is able to exploit the dependence between all variables in a joint kernel. The average gains for PB2-Mix over PB2-Rand are statistically significant, with $p=0.044$ in Welch's t-test. 

\begin{figure}[H]
    \centering
    \vspace{-3mm}
    \begin{minipage}{0.99\textwidth}
    \centering\subfigure[Hyperparameter impact by augmentation type]{\includegraphics[width=.57\linewidth]{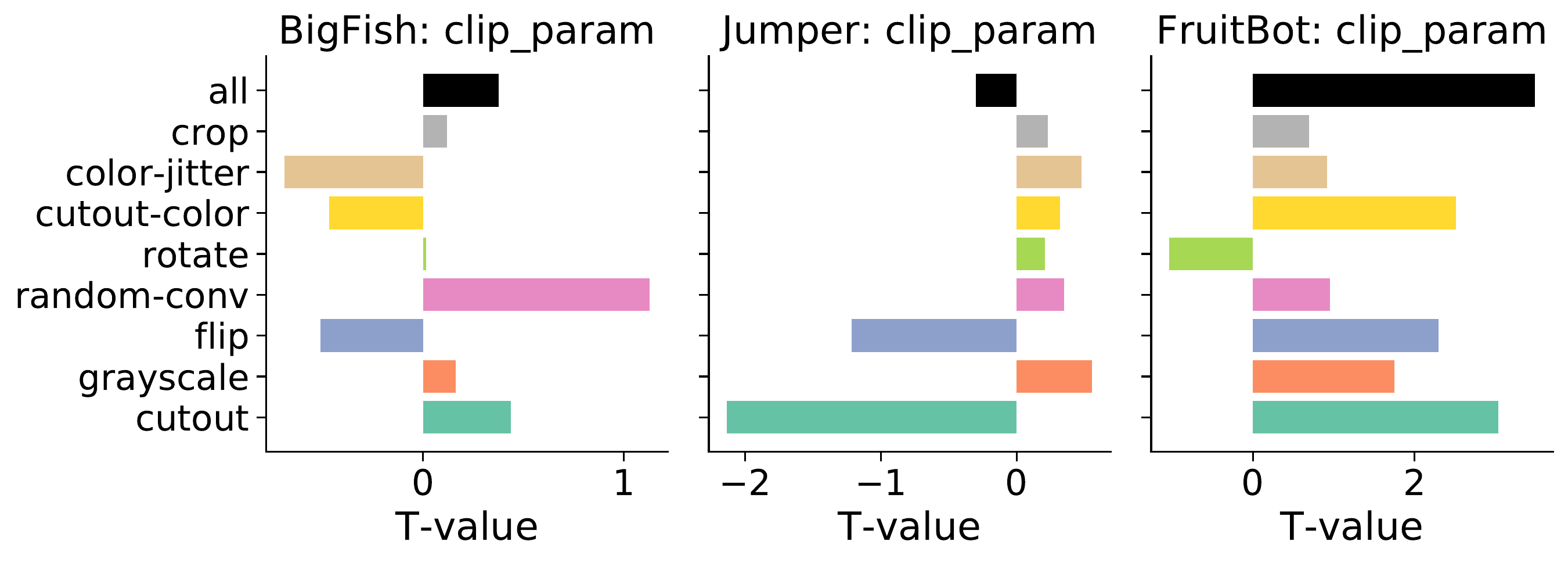}}
    \centering\subfigure[Augmentation selection efficacy]{\includegraphics[width=.42\linewidth]{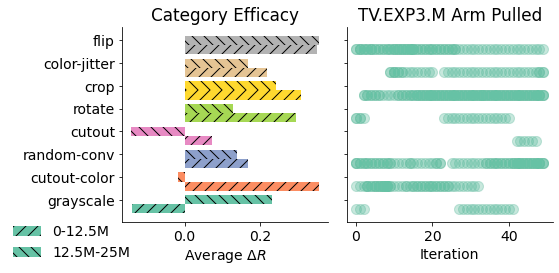}}
    \vspace{-4mm}
    \caption{\small{(a) T-values for univariate regressions predicting reward change given a continuous hyperparameter, conditioned on category. (b) Left: Mean reward change for each category in the BigFish environment, across all experiments, in the first and second half of training. Right: All arms selected by TV.EXP3.M.}}
    \label{figure:tvexp3_effectiveness}
    \vspace{-5mm}
    \end{minipage}
\end{figure}

\textbf{Do we need to model dependence?} To explore the dependence of the parameters in the Procgen setting, we tested relationship between each continuous hyperparameter and the subsequent change in reward ($f_t(.)$). For each \emph{individual} hyperparameter, we \emph{conditioned the data on the category} currently being used, and fit a linear regression model. In Fig. \ref{figure:tvexp3_effectiveness}.a) we show the t-values for three separate examples, with the full grid in the Appendix (Fig. \ref{figure:reg_tstats}, Section \ref{sec:addition_experiments}). As we see, the relationship between continuous variables and learning performance varies depending on the category selected, confirming a dependence in Procgen \cite{hyp_dep_gen}. In particular, we see for BigFish and Jumper the clip parameter relationship with training performance is heavily dependent on the data augmentation type used. Interestingly we also include an example from FruitBot, where the dependence seems to be weaker as the clip param is positively related with improving policies for all but one category. This is likely the reason for relatively stronger performance for PB2-Rand in FruitBot.

We also investigate the effectiveness of TV.EXP3.M in selecting the data augmentation type. In Fig. \ref{figure:tvexp3_effectiveness}.b) we show the mean change in reward for each category in the BigFish environment, as well as all the categories selected by the new PB2 variants using TV.EXP3.M. We see that the three most selected augmentations (flip, crop and color-jitter) all lead to positive rewards. Interestingly, we also see that cutout-color is frequently selected at the \emph{beginning}, but not used at all at the end where it no longer improves training, thus demonstrating TV.EXP3.M is able to adapt effectively.


\begin{wrapfigure}{R}{0.23\textwidth}
    \centering
    \vspace{-6mm}
    \centering\subfigure{\includegraphics[width=0.98\linewidth]{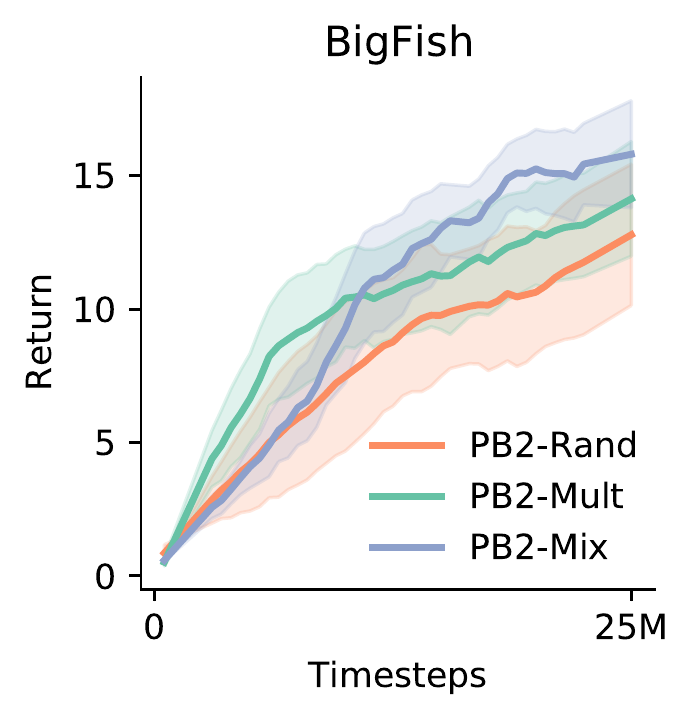}}
    \vspace{-4mm}
    \caption{\small{Mean test performance, sem shaded.}}
    \vspace{-2mm}
    \label{figure:bigfish12}
\end{wrapfigure}

\textbf{Does it scale?} The key benefit of PBT algorithms is the ability to make use of large population sizes with the same wall clock time of a single agent. Thus far, we have only required four agents to achieve strong performance on the Procgen tasks considered, but the question remains: how do these new methods scale? To test this we train with a population size of $12$ for PB2-Rand, PB2-Mult and PB2-Mix on the BigFish task. We repeat the experiment for three seeds and report the mean test return from the agent with the best training performance in \ref{figure:bigfish12}. As we see, PB2-Mix remains the strongest method, with final performance of $16\pm1.15$. This is comparable to the state of the art performance  \cite{jiang2020prioritized, raileanu2021decoupling}. Since our method is orthogonal to these, it is likely combining them could provide further gains.

\subsection{Discussion} 
\label{sec:discussion}

Overall our results make a strong case for using population based AutoRL methods on the Procgen benchmark, with further improvement as both continuous (PB2-Rand) and categorical (PB2-Mix, PB2-Mult) variables are selected efficiently. We introduced a new class of hierarchical methods with two variations, each with their own strengths and weaknesses. Firstly, PB2-Mult, which has the strongest theoretical results and performance on the synthetic task, was less impressive in Procgen. We hypothesize that the continuous/categorical dependence may not be strong enough to justify splitting the dataset, which may mean we do not have enough data with only four agents and eight categories. Indeed, the performance for PB2-Mult may be superior when using greater population sizes or smaller frequency of updates.  

Conversely, PB2-Mix provides the ability to automatically trade-off this dependence, and the experimental results show an improvement over the PB2-Rand baseline in almost every setting. However, the limitation of this approach is that optimizing the mix parameter at each explore step does introduce additional computational complexity, which increases the runtime proportionally to the amount of data. We note the runtime can be improved in several ways, such as by fixing the kernel hyperparameters, using a sliding window of data, or using a faster GP implementation (an active area of research). 

Indeed, there is inevitably a trade-off in terms of the desired efficiency of the explore step and the computational cost one is willing to incur. When the RL experiments are expensive to run, for example with vision based inputs, the explore steps (measured in a few seconds) make up a trivial amount of time compared to the several hours of training. However in the opposite setting where evaluations are extremely cheap, for example using a fast simulator with a large compute cluster, it may even be preferable to use the original PBT since there could be enough samples to explore the space with random search. We believe in most cases PB2-Mix would be the best approach, which is supported by our experiments. 

Finally we note another potential area for improvement in all the algorithms in this paper (and AutoRL in general): we make decisions in the explore/exploit step based on \emph{training} performance (on changing levels), ignoring generalization ability. This may lead to accidentally removing the best agents on the test set, while it may also lead to inaccuracy in our GP models. One potentially promising approach may be to consider a fixed set of ``validation levels'', learned using techniques such as \emph{Prioritized Level Replay} \cite{jiang2020prioritized}, but we leave this to future work.

\section{Conclusion and Future Work}
\label{sec:conclusion}

In this paper we expand the capabilities of population based bandits (PB2) by making it possible to efficiently select \emph{both} continuous and categorical hyperparameters. We introduced a new time varying multi-armed bandit algorithm for selecting categorical hyperparameters, and presented a hierarchical approach to subsequently select continuous parameters. We believe this work is an important step in increasing the capability of population based approaches for AutoRL. This should provide benefits at both ends of the spectrum: it can make RL accessible to a wider audience by significantly lowering the cost of tuning hyperparameters, while also potentially leading to stronger performance for those with larger resources. We also showed the effectiveness of learning both hyperparameters and data augmentation on the fly, for challenging procedurally generated environments. We are not the first to propose population based methods in this setting \cite{neuroprocgen}, but we provided strong evidence that PBT-style hyperparameter tuning can provide significant benefits.

We are particularly excited by many future directions from here. The most natural is to learn more: can we learn policy architectures in the explore step, for example making use of kernels between networks \cite{wan2021interpretable}? Can we learn algorithms altogether, extending exciting recent work \cite{learnedPG, evolving_algos}? For the specific case of Procgen environments, can we jointly learn hyperparameters, augmentation and prioritize which levels to train on \cite{jiang2020prioritized}? In addition, since we have a population of agents, can we improve performance by sharing experience across agents \cite{franke2021sampleefficient}? Can we benefit from encouraging these agents to be diverse with respect to one another \cite{dvd, qdnature}? We are also curious to see whether recent innovations in other areas of AutoML can be combined in a PBT-style framework \cite{automl_zero}. We think the combination of these ideas could lead to dramatic gains in performance for RL agents. 


\section*{Acknowledgements}

The experiments in this paper were conducted using AWS. The authors would like to thank Roberta Raileanu for providing open source code and both Roberta and Minqi Jiang for discussion on reporting for Procgen results, as well as Xingyou Song and Yingjie Miao for discussion around AutoRL. Finally, this work was improved thanks to constructive feedback from anonymous reviewers.

\bibliographystyle{abbrv}
\bibliography{refs}

\newpage
\appendix

\newpage

\section{Implementation Details}
\label{sec:implementation_details}
\subsection{Hyperparameter Ranges}

\textbf{RL Hyperparameters}: Below are the continuous hyperparameters optimized as well as the fixed parameters, which were all taken from \cite{ucb_drac}. For the categorical variables, we use the data augmentations used in \cite{ucb_drac, rad, drq}: \textit{crop, grayscale, cutout, cutout-color, flip, rotate, random convolution and color-jitter}. 

\begin{table}[h!]
    \centering
    \begin{minipage}{.45\linewidth}
        \caption{\footnotesize{DrAC Learned Hyperparameters}}
        \label{table:drac_learned}
        \centering
        \scalebox{0.85}{
        \begin{tabular}{l*{2}{c}r}
        \toprule
        \textbf{Parameter} & \textbf{Value}  \\
        \midrule
        PPO Clip ($\epsilon$) & $[0.01, 0.5]$  \\
        Learning Rate & $[10^{-5}, 10^{-3}]$  \\
        Entropy Coeff & $[0, 0.2]$ \\
        Regularization Parameter ($\alpha_r$) & $[0.01, 0.5]$ \\
        \bottomrule
        \\
    \end{tabular}}
    \end{minipage}
    \begin{minipage}{.45\linewidth}
        \caption{\footnotesize{DrAC Fixed Hyperparameters}}
        \label{table:drac_fixed}
        \centering
        \scalebox{0.85}{
        \begin{tabular}{l*{2}{c}r}
        \toprule
        \textbf{Parameter} & \textbf{Value}  \\
        \midrule
        $\gamma$ & 0.999  \\
        $\lambda$ & 0.95  \\
        \# timesteps per rollout & 256 \\
        \# epochs per rollout & 3 \\
        \# minibatches per epoch & 8 \\
        optimizer & Adam \\
        
        \bottomrule
        \\
    \end{tabular}}
    \end{minipage}
\end{table}

\textbf{PB2 Hyperparameters}: We use the same UCB hyperparameters as in \cite{pb2}, which come from \cite{bogunovic2016time}. Concretely, we set $c1=0.2$, $c2=0.4$. We believe performance may be increased by selecting these parameters more carefully. However, it is promising to see that this fine tuning is not crucial for performance.

\subsection{Environment Details}

We used five games from the Procgen environment \cite{procgen}. We chose these games as there appeared to be a greater degree of variance in the optimal data augmentation in \cite{ucb_drac}. All protocols for the environment are from \cite{ucb_drac}, as we followed the author's open source implementation for agent training\footnote{see: \url{https://github.com/rraileanu/auto-drac}}, and only varied the hyperparameters and data augmentation chosen. 

\begin{figure}[H]
    \centering
    \begin{minipage}{0.99\textwidth}
    \centering\subfigure[\textbf{BigFish}]{\includegraphics[width=.22\linewidth]{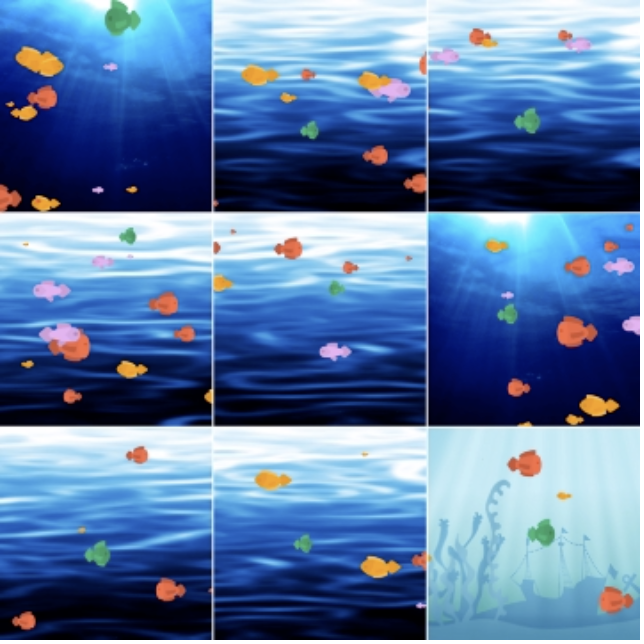}}    
    \centering\subfigure[\textbf{CaveFlyer}]{\includegraphics[width=.22\linewidth]{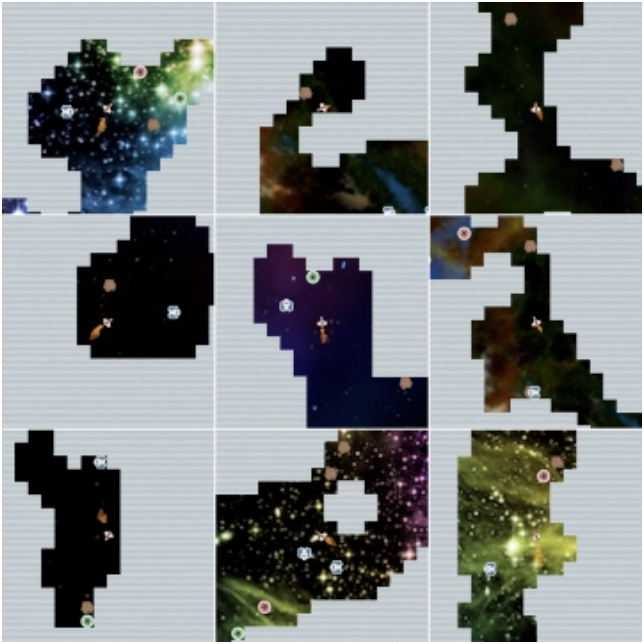}}    
    \centering\subfigure[\textbf{CoinRun}]{\includegraphics[width=.22\linewidth]{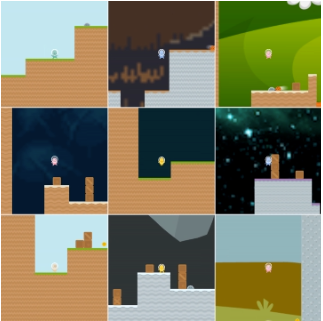}}
    \centering\subfigure[\textbf{FruitBot}]{\includegraphics[width=.22\linewidth]{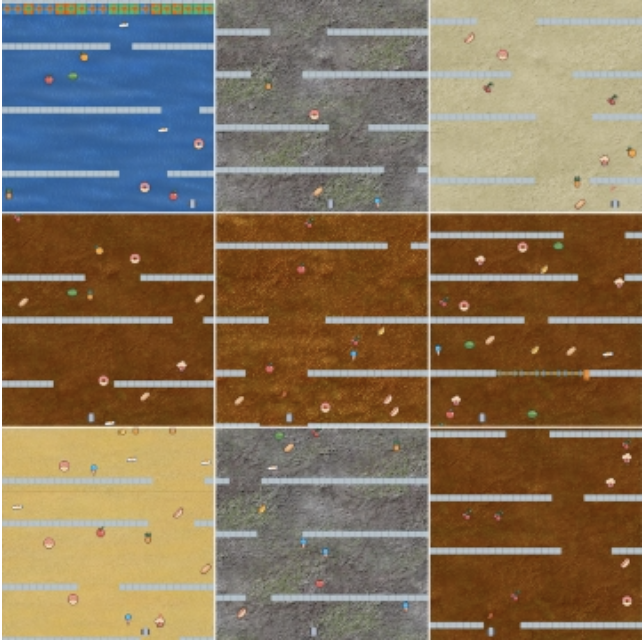}}
    \centering\subfigure[\textbf{Jumper}]{\includegraphics[width=.22\linewidth]{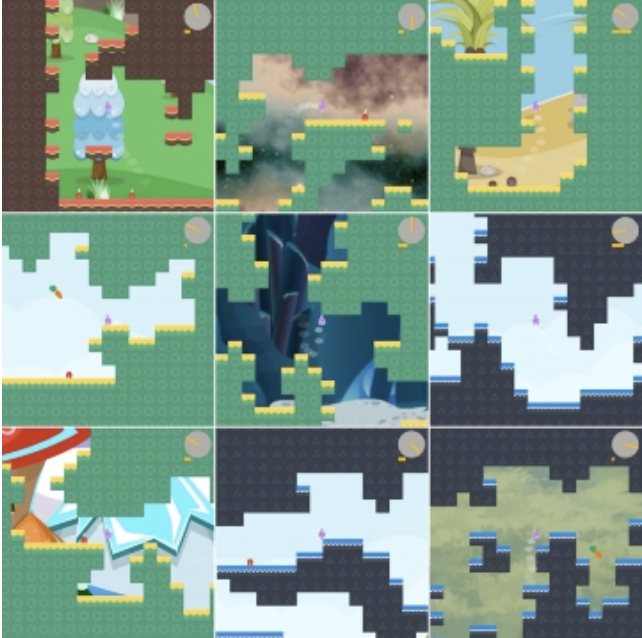}} 
    \centering\subfigure[\textbf{Leaper}]{\includegraphics[width=.22\linewidth]{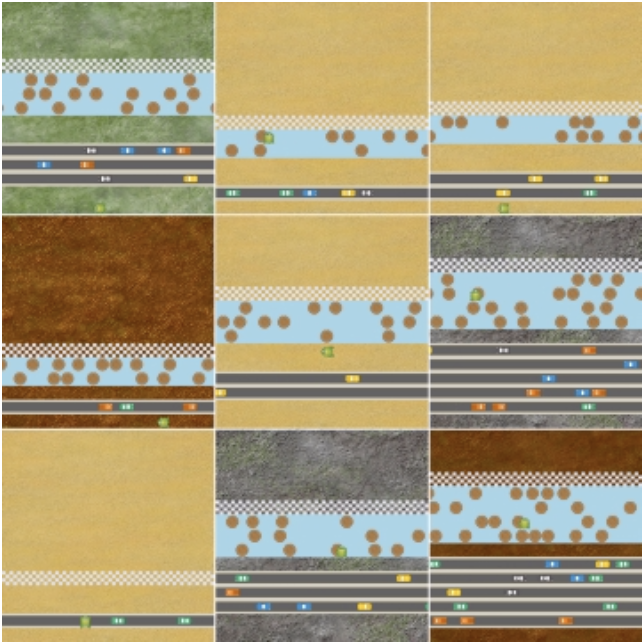}} 
    \centering\subfigure[\textbf{StarPilot}]{\includegraphics[width=.22\linewidth]{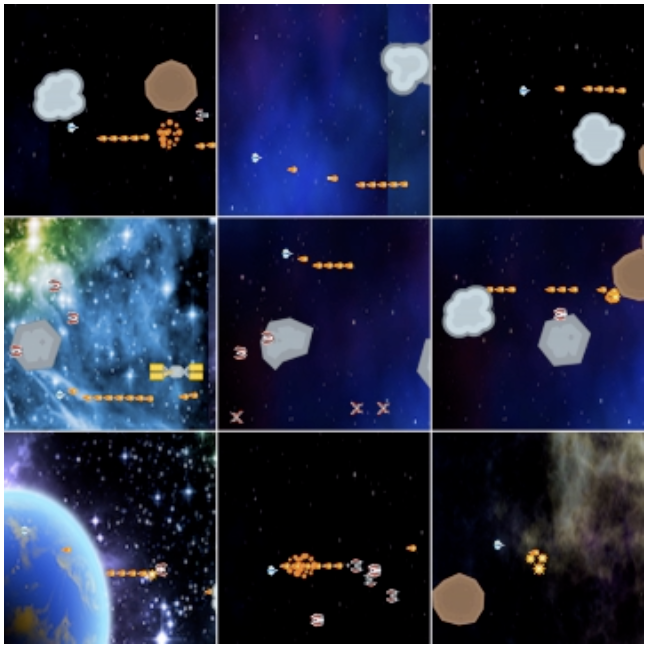}} 
    \caption{\small{Example levels from the seven games considered from \cite{procgen}}}
    \label{figure:env_images}
    \end{minipage}
\end{figure}

\subsection{Infrastructure Details}

Each trial was run on a single GPU, taking around three days to train all four agents. This was due to the nature of our existing compute infrastructure and could likely be made significantly faster with parallelization. The existing code is thus not optimized to work in a distributed fashion, however it will be open sourced in a similar structure to the open source PB2-Rand implementation (part of the Ray library \cite{ray}). Since we followed the open source version of PB2-Rand, it is a trivial extension to include PB2-Mult and PB2-Mix, using new kernels provided both in the supplementary material (as part of the research code) or the synthetic function notebook (see the link in Section \ref{sec:experiments}). 

\section{Additional Results}
\label{sec:addition_experiments}

\begin{table}[H]
\vspace{-2mm}
\begin{center}
\caption{\small{\textbf{Train} performance for seven Procgen games. $\dagger$ indicates training was conducted for a single agent for 25M timesteps. Hyperparameters were initialized at optimized values, meaning the effective number of timesteps is much higher. Final performance is the average of 100 trials and results were taken from \cite{ucb_drac}. $\ddagger$ indicates training was conducted by a population of four agents for 25M timesteps each, with several hyperparameters initialized at random. Each agent is evaluated for $100$ trials on train and test levels and we present the mean train performance of the agent with the best training performance. 
}}
\label{table:procgen_results_train}
\scalebox{0.82}{
\begin{tabular}{ l | cc | cc | cc } 
\toprule
\textbf{Environment} & PPO$\dagger$  & UCB-DrAC$\dagger$ & PBT$\ddagger$ & PB2-Rand$\ddagger$ & PB2-Mult$\ddagger$ & PB2-Mix$\ddagger$  \\ 
\midrule 
BigFish & $8.9 \pm 1.5$ & $13.2\pm2.2$ & $12.0\pm3.3$ & $20.8\pm1.8$ & $18.2\pm2.0$ & $17.1\pm2.9$ \\ 
CaveFlyer & $6.8\pm0.6$ & $5.7\pm0.6$ & $7.2\pm0.4$ & $7.3\pm1.7$ & $7.0\pm1.3$ & $7.5\pm0.9$ \\ 
CoinRun & $9.3\pm0.3$ & $9.5\pm0.3$ & $8.2\pm0.8$ & $10.0\pm0$ & $9.9\pm0.2$ & $9.6\pm0.2$ \\ 
FruitBot & $29.1\pm1.1$ & $29.5\pm1.2$ & $27.4\pm0.2$ & $32.2\pm0.7$ & $30.5\pm1.4$ & $30.9\pm1.2$ \\ 
Jumper & $8.3\pm0.4$ & $8.1\pm0.7$ & $8.3\pm0.2$ & $9.0\pm0.2$ & $8.9\pm0.1$ & $9.2\pm0.5$ \\ 
Leaper & $5.5\pm0.4$ & $5.3\pm0.5$ & $4.1\pm0.5$ & $6.9\pm2.0$ & $5.6\pm2.1$ & $7.1\pm2.3$ \\ 
StarPilot & $29.8\pm2.3$ & $35.3\pm2.2$ & $33.6\pm7.4$ & $44.1\pm2.7$ & $40.3\pm1.3$ & $41.8\pm2.7$ \\ 
\bottomrule
\end{tabular}}
\end{center}
\vspace{-5mm}
\end{table}

\begin{figure}[H]
    \centering
    \vspace{-2mm}
    \begin{minipage}{0.99\textwidth}
    \subfigure{\includegraphics[width=.99\linewidth]{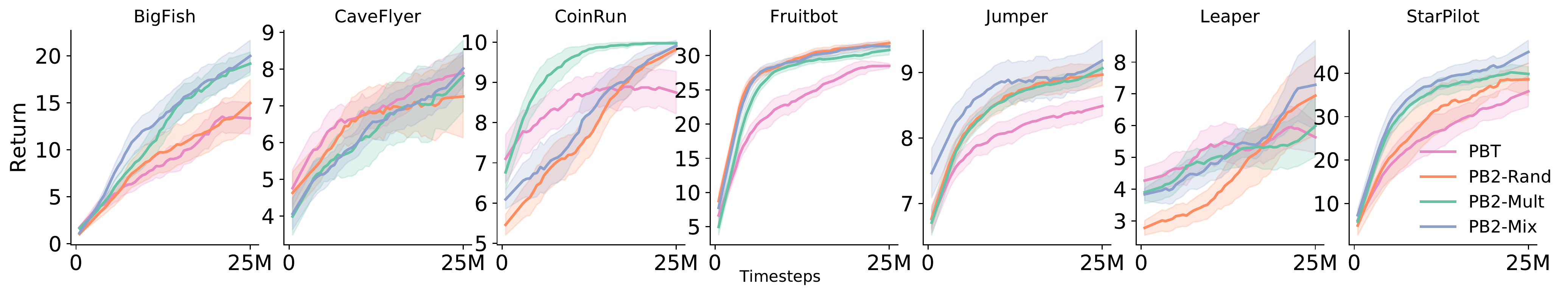}} 
    \vspace{-3mm}
    \caption{\small{\textbf{Train Performance}: Learning curves for all seven Procgen games. Plots show the mean $\pm1$sem for the best training performance within the population. Results are averaged over $5$ seeds.}}
    \label{figure:procgen_train}
    \end{minipage}
\end{figure}

In Table \ref{table:procgen_results_train} we show training performance for the agents shown in the main paper. Interestingly, we see strong training performance for PB2-Rand, indicating a larger generalization gap. We also show training learning curves for all algorithms in Fig. \ref{figure:procgen_train}, where we once again report the best agent at each timestep.

\newpage
\textbf{Is there dependence between data augmentation and hyperparameters in Procgen?} Next we consider the dependence between the hyperparameters. For each environment, we fit a linear regression model predicting the change in reward for $500$k steps of training (i.e. one step of PB2), with the independent variable being \textbf{one} of the continuous hyperparameters. Colors correspond to the setting where we condition on an individual category (by creating a subset), while black corresponds to using the entire dataset. We use the data from all trials and show the results in In Fig. \ref{figure:reg_tstats}.

\begin{figure}[H]
    \centering
    \begin{minipage}{0.99\textwidth}
    \vspace{-2mm}
    \subfigure{\includegraphics[width=.97\linewidth]{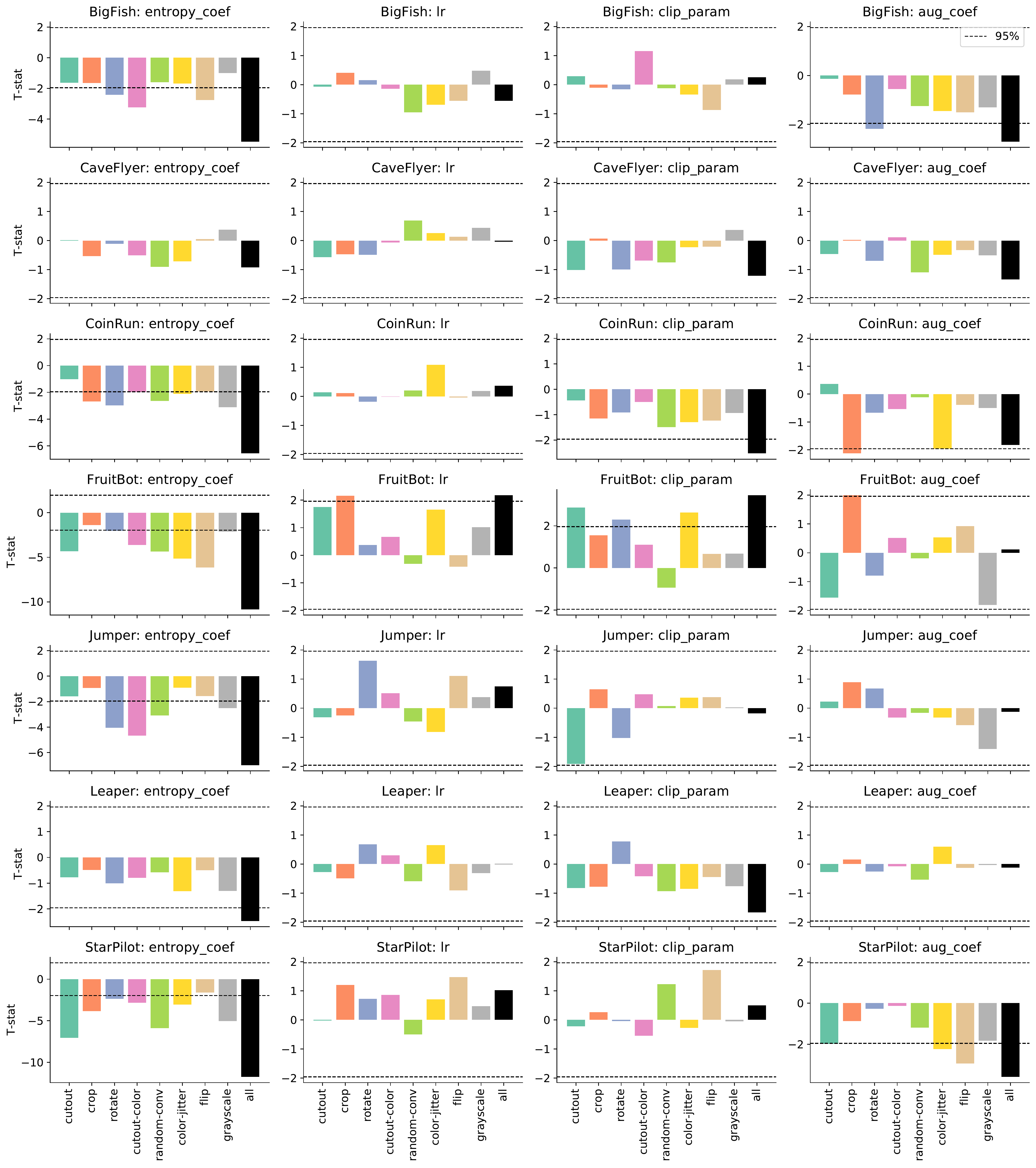}} 
    \vspace{-4mm}
    \caption{\small{t-statistic for the independent variable (one continuous hyperparameter, in the title of each plot) in predicting the change in reward for the given environment. Blue columns represent the data being from one. The dotted line indicates $p<0.05$. }}
    \vspace{-5mm}
    \label{figure:reg_tstats}
    \end{minipage}
\end{figure}

As we see, in some cases the category does not have any impact on the relationship, for example the entropy coefficient for CoinRun: an increase in entropy appears to be negatively related with training performance across all augmentations. However, for some settings such as the clip parameter in BigFish, we see large swings in the t-stat depending on category. We even see some settings where the aggregate relationship (grey) is very small, but conditioning on the category makes it worthwhile to tune the hyperparameter, for example the learning rate for CaveFlyer. In this case it may be challenging for PB2-Rand to achieve gains from tuning the learning rate. We note  this analysis is purely illustrative and has many assumptions, in particular 1) we use a linear univariate model 2) we ignore the time-varying component.

\newpage
\textbf{Learned Schedules} The next pages contain learned schedules from all agents trained by our algorithms and baselines. In each case we show all agents for all seeds, with blue corresponding to a config that contributed to the final best (training) agent. The most notable observation across all environments is the stark contrast between PBT and PB2-based methods. PBT clearly increments parameters, and gradually shifts during training, while PB2 rapidly explores the boundaries. We note in some cases that PB2-Mult never re-explores the middle regions, which we hypothesize is due to the lack of data when creating separate models for each of the eight categories.

\begin{figure}[H]
    \vspace{-3mm}
    \centering
    \begin{minipage}{0.99\textwidth}
    \subfigure{\includegraphics[width=.99\linewidth]{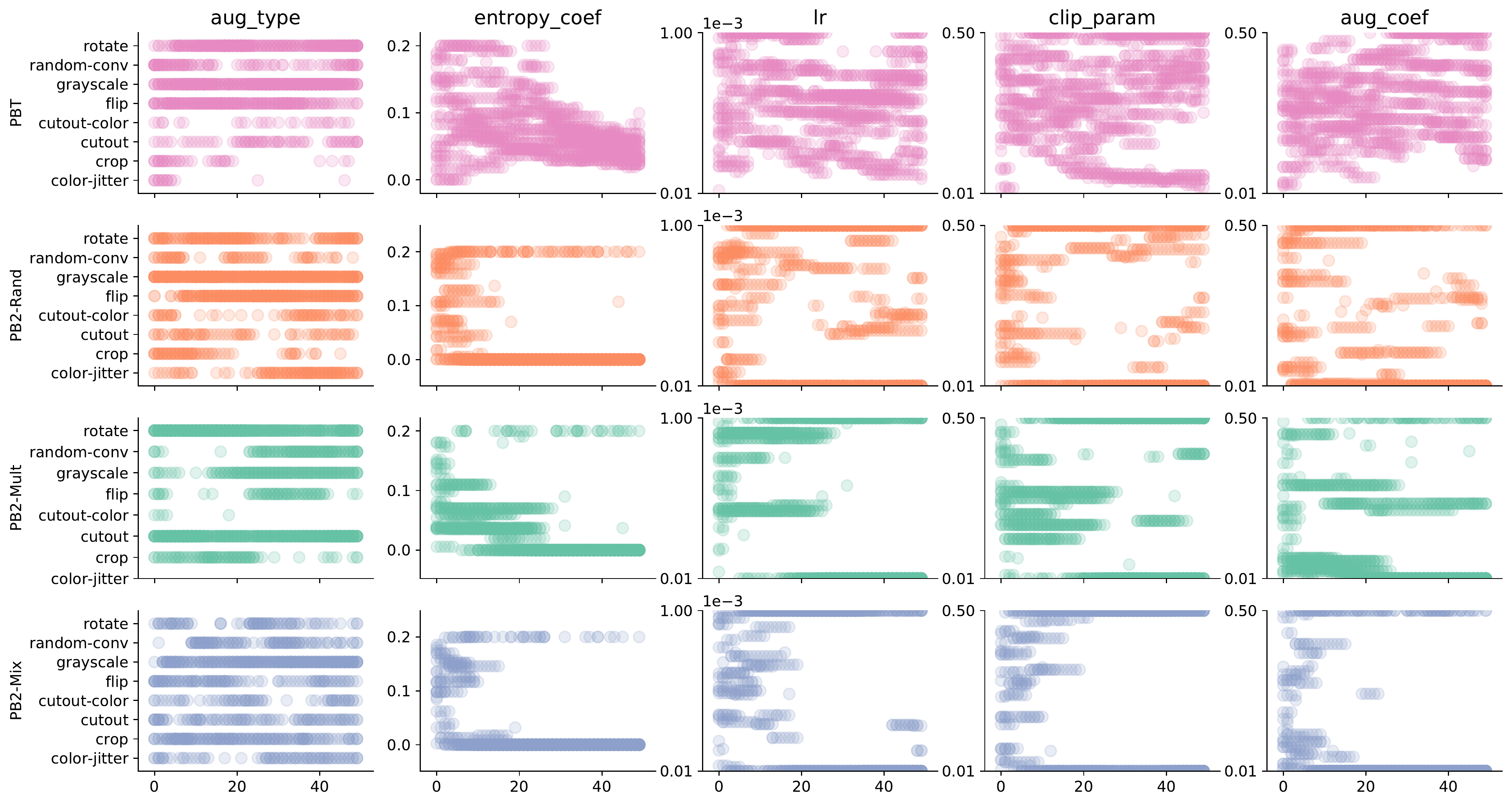}}  
    \caption{\small{Learned schedules for BigFish by algorithm, for all seeds. Each point corresponds to one agent.}}
    \label{figure:schedule_bigfish}
    \end{minipage}
\end{figure}
\begin{figure}[H]
    \centering
    \begin{minipage}{0.99\textwidth}
    \subfigure{\includegraphics[width=.99\linewidth]{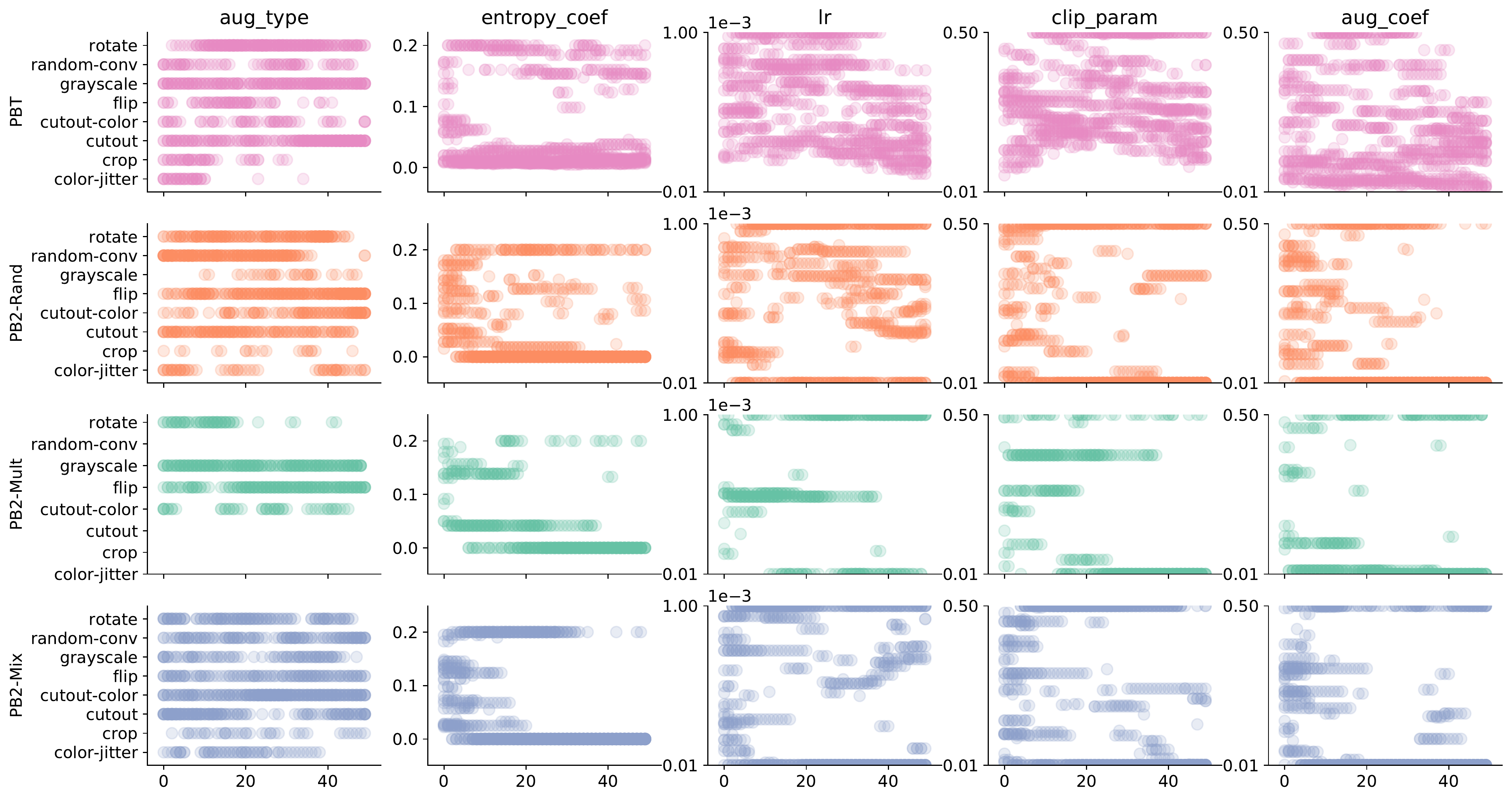}}  
    \caption{\small{Learned schedules for CaveFlyer by algorithm, for all seeds. Each point corresponds to one agent.}}
    \label{figure:schedule_caveflyer}
    \end{minipage}
\end{figure}

\begin{figure}[H]
    \centering
    \begin{minipage}{0.99\textwidth}
    \subfigure{\includegraphics[width=.99\linewidth]{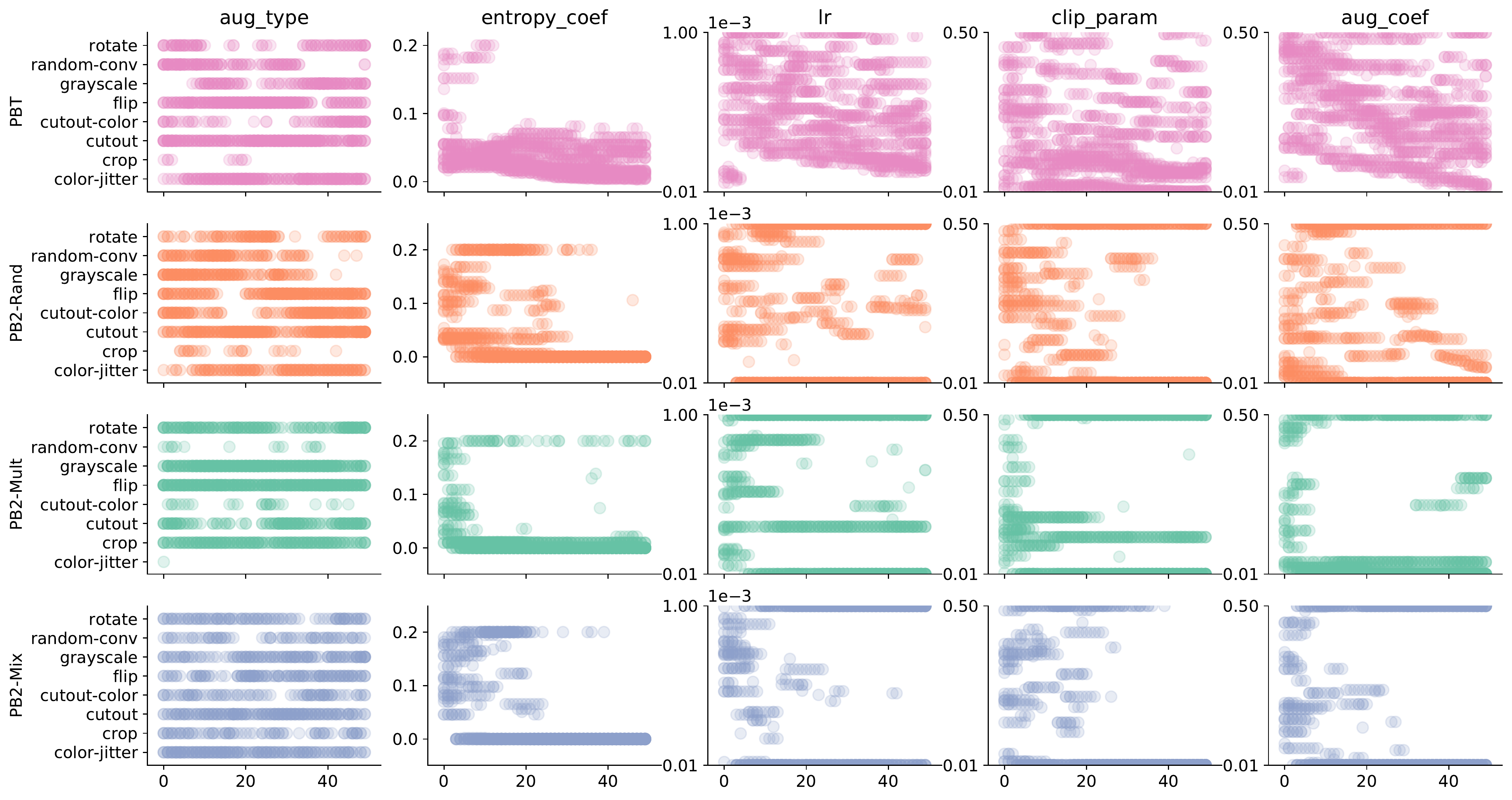}}  
    \caption{\small{Learned schedules for CoinRun by algorithm, for all seeds. Each point corresponds to one agent.}}
    \label{figure:schedule_coinrun}
    \end{minipage}
\end{figure}

\begin{figure}[H]
    \centering
    \vspace{-20mm}
    \begin{minipage}{0.99\textwidth}
    \subfigure{\includegraphics[width=.99\linewidth]{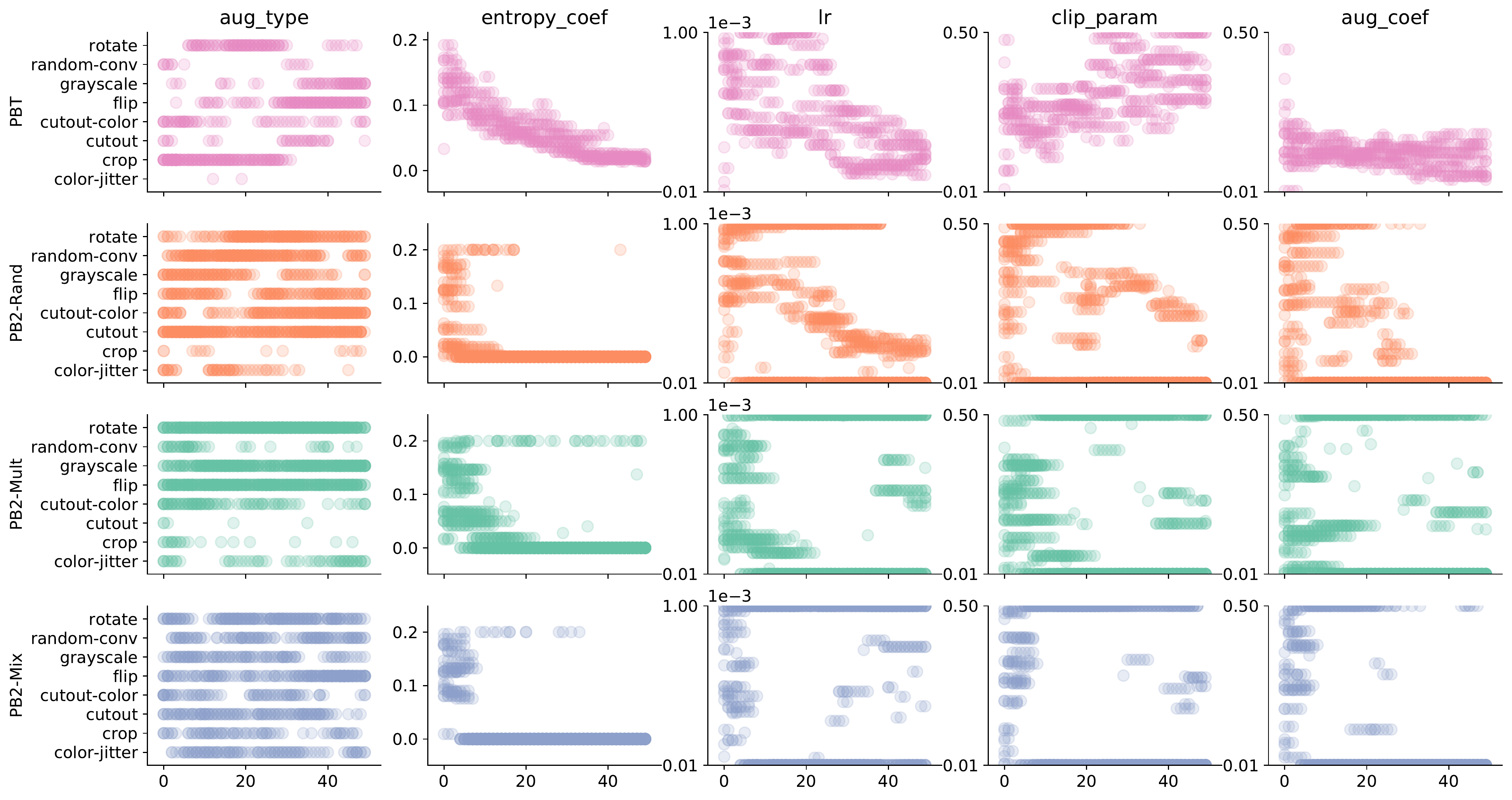}}  
    \caption{\small{Learned schedules for FruitBot by algorithm, for all seeds. Each point corresponds to one agent.}}
    \label{figure:schedule_fruitbot}
    \end{minipage}
\end{figure}

\begin{figure}[H]
    \centering
    \begin{minipage}{0.99\textwidth}
    \subfigure{\includegraphics[width=.99\linewidth]{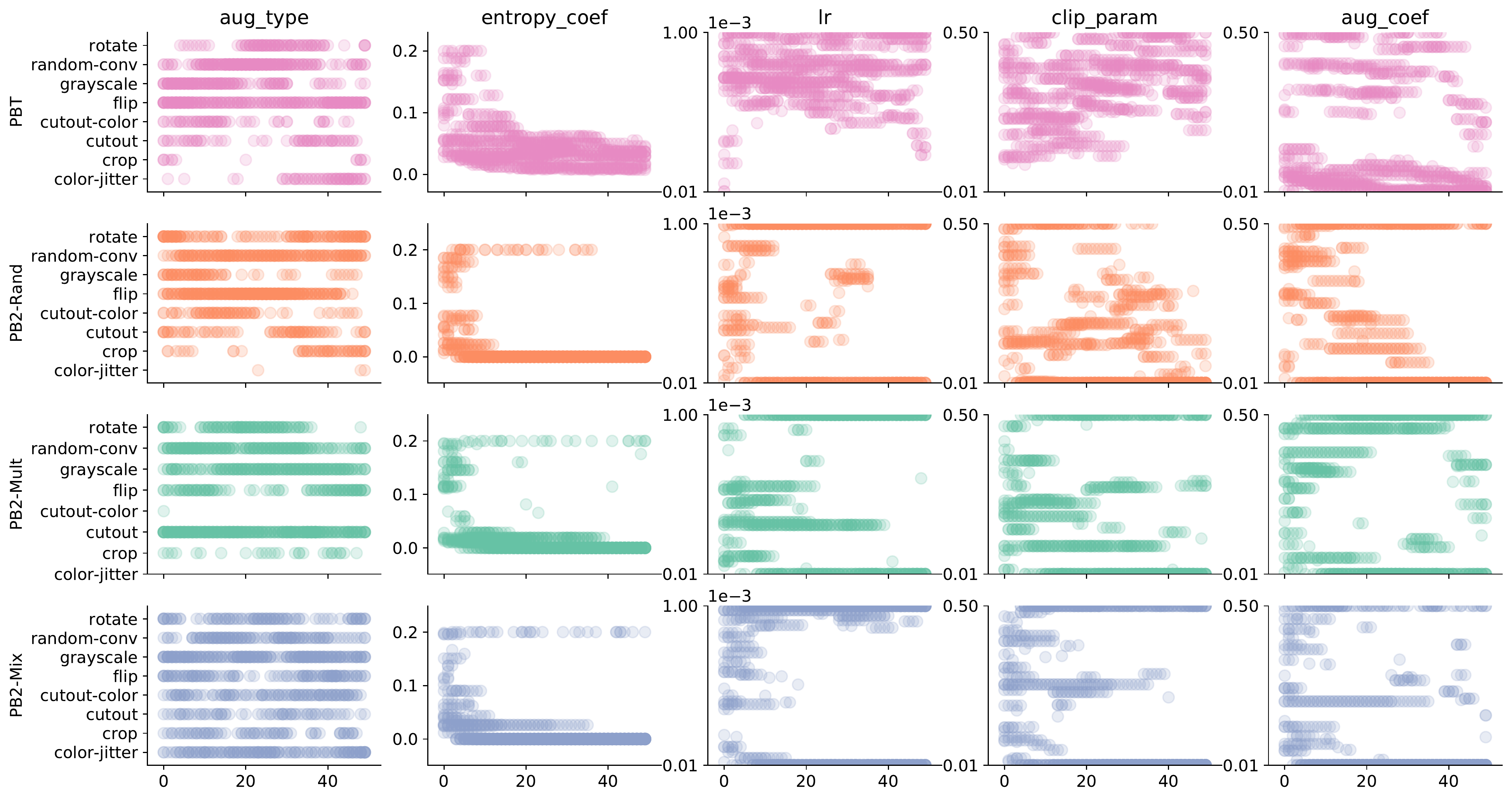}}  
    \caption{\small{Learned schedules for Jumper by algorithm, for all seeds. Each point corresponds to one agent.}}
    \label{figure:schedule_jumper}
    \end{minipage}
\end{figure}

\begin{figure}[H]
    \centering
    \vspace{-20mm}
    \begin{minipage}{0.99\textwidth}
    \subfigure{\includegraphics[width=.99\linewidth]{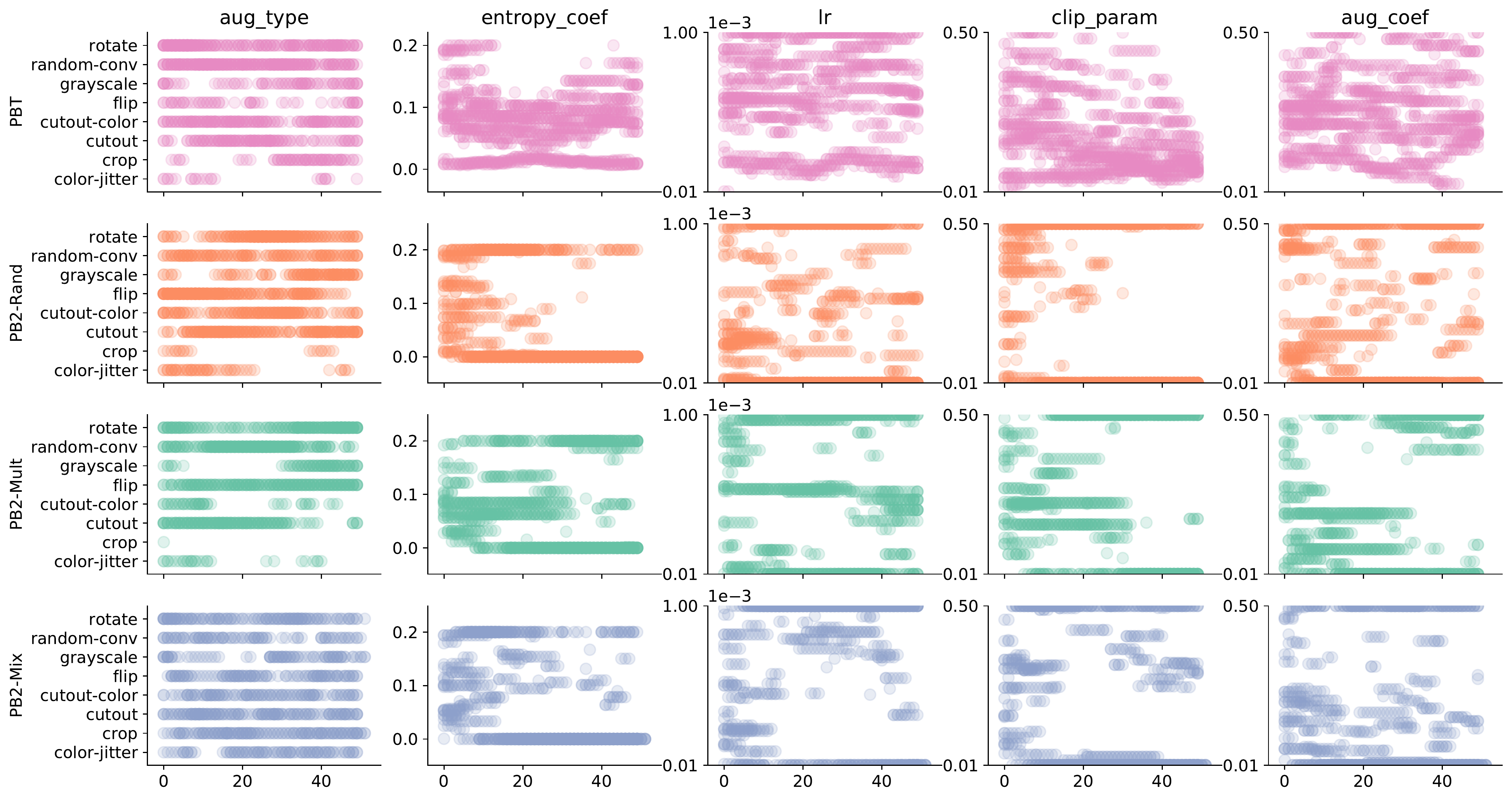}}  
    \caption{\small{Learned schedules for Leaper by algorithm, for all seeds. Each point corresponds to one agent.}}
    \label{figure:schedule_leaper}
    \end{minipage}
\end{figure}

\begin{figure}[H]
    \centering
    \begin{minipage}{0.99\textwidth}
    \subfigure{\includegraphics[width=.99\linewidth]{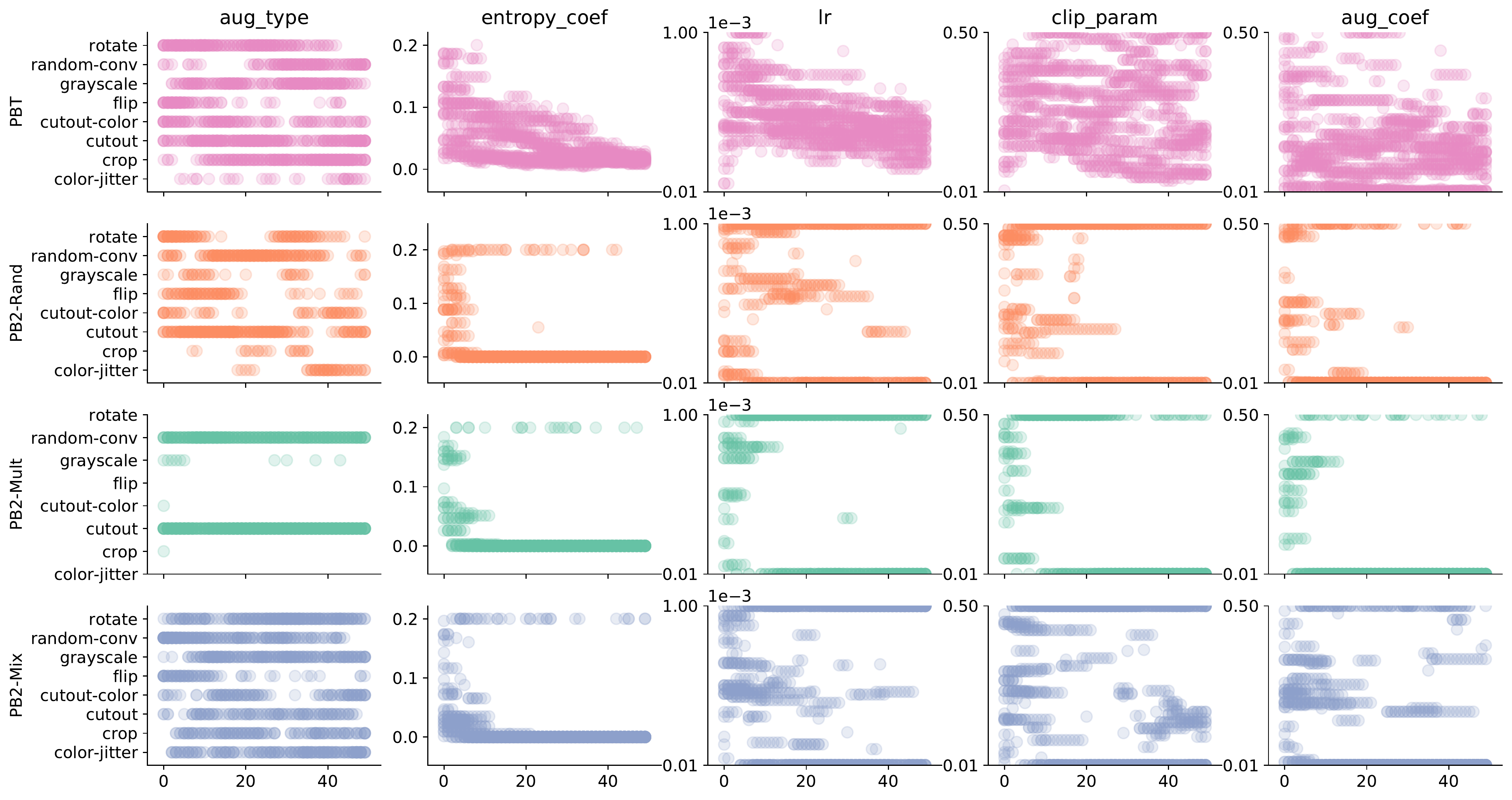}}  
    \caption{\small{Learned schedules for StarPilot by algorithm, for all seeds. Each point corresponds to one agent.}}
    \label{figure:schedule_starpilot}
    \end{minipage}
\end{figure}

\newpage
\section{Additional Background}

\subsection{DrAC}

In recent times, there has been increased interest in \emph{generalization} of RL agents, after multiple works showed many RL agents are simply overfitting to a single deterministic training environment \cite{packer2019assessing, coinrun}. One approach to produce more generalizable agents is data augmentation, which has recently shown impressive results \cite{dataaug_zs, drq, rad, curl}. To formalize the data augmentation step, we follow \cite{drq} and define an optimality-invariant state transformation $f:\mathcal{S} \times \mathcal{H} \rightarrow \mathcal{S}$ as a mapping that preserves both the policy $\pi$ and value function $V$, i.e. $V(s) = V(f(s, \nu))$ and $\pi(a|s) = \pi(a|f(s,\nu)), \forall s\in\mathcal{S}, \nu \in \mathcal{H}$, where $\nu$ are parameters of $f(.)$ drawn from the set of possible parameters $\mathcal{H}$. In this work we focus on the formulation from \cite{ucb_drac} who propose additional loss terms for regularizing the policy and value function:
\begin{align}
     G_\pi = \text{KL}[\pi_\theta(a|f(s,\nu])|\pi(a|s)], 
\end{align}
\begin{align}
     G_V=(V_\theta(f(s,\nu))-V(s))^2
\end{align}
Combining the PPO objective with $G_\pi$ and $G_V$ produces the \emph{data-regularized actor critic} or DrAC objective as follows:
\begin{align}
    \mathcal{L}_{\mathrm{DrAC}}(\theta) = \mathcal{L}_{\mathrm{PPO}}(\theta) - \alpha_r(G_\pi + G_V)
\end{align}
where $\alpha_r$ is the weight of the regularization term, another hyperparameter to consider. The results in \cite{ucb_drac} show that learning which data augmentation function $f$ to use (from a fixed set) and grid searching over the $\alpha_r$ parameter can achieve new state-of-the-art results in challenging Procgen environments. 
\newpage

\section{Theoretical Results}

We derive the theoretical proofs presented in the main paper.

\subsection{Time-varying EXP3 Multiple play (\textsc{tv.exp3.m})} \label{appendix_subsec:tv_exp3_m}

We present a new algorithm for parallel (or multiple play) multi-armed bandits in the time-varying
setting with adversarial feedback. Particularly, we extend the
multiple play EXP3 algorithm (or EXP3.M) \cite{uchiya2010algorithms},
to the time-varying setting where the unknown reward distribution of each arm can change arbitrarily, but the total number of change points is no more than $V$ times. We refer to Table \ref{table:notation} for the notations used in the proofs.

We summarize the \textsc{tv.exp3.m} in Algorithm \ref{alg:EXP3_MultiplePlay_S} -- this is complementary to the brief Algorithm \ref{alg:EXP3_MultiplePlay_S_short} in the main paper. In Algorithm \ref{alg:EXP3_MultiplePlay_S}, we maintain a set of probability vectors (step 9) for each arm which will be specified and weighted in the optimal way derived by the theory. Then, at each round we select a batch of $B$ arms for parallel evaluations (step 10), observe the reward (step 11) and update the model (step 12,13,14). The normalization steps 3-7 are to prevent from being biased toward a single arm which performs overwhelmingly well, thus overly exploiting.

\begin{table}
\caption{Notation used in the theoretical analysis.\label{table:notation}}
\begin{centering}
\scalebox{0.8}{

\begin{tabular}{ccc}
\toprule 
\textbf{Variable} & \textbf{Domain} & \textbf{Meaning} \tabularnewline
\midrule 
$C$ & $ \mathcal{N}$ & number of categories (number of arms)\tabularnewline
\midrule 
$T$ & $\mathcal{N}$ & maximum number of bandit update (the number of $t_{\mathrm{ready}}$  in PB2) \tabularnewline
\midrule 
$V$ & $\mathcal{N}$ & number of time segments, how many times the function has been shifted\tabularnewline
\midrule 
$B$ & $\mathcal{N}$ & batch size (number of parallel agents)\tabularnewline
\midrule 
$S_t$ & $list$ & a list of $B$ selected categories $[ c_{t,1}, c_{t,2},...,c_{t,B}]$\tabularnewline
\midrule 
$e$ & $\mathcal{R^+}$ & this is Euler's number $2.71828$\tabularnewline
\midrule 
$\left[p_{t}^{1},...p_{t}^{C}\right]$ &  $list$ & probability vector at iteration $t$ and category $c=1....C$\tabularnewline
\midrule 
$w_{c}$ & $\mathcal{R}$ & weight for categorical $c$ (or arm $c$)\tabularnewline
\midrule 
$W_{t}=\sum_{c=1}^{C}w_{c}$ & $\mathcal{R^+}$ & sum of the weight vector at iteration $t$ \tabularnewline
\midrule 
$c_{t,b}\in\{1...C\}$ & $\mathcal{N}$ & a selection at iteration $t$ at agent $b$\tabularnewline
\midrule 
$c_{t} =[c_{t,1},...,c_{t,B}] $ & $list$ & a batch of $B$ categorical selections at iteration $t$  \tabularnewline
\midrule 
$c_{t}^{*}\in\{1...C\}$ & $list$ & an optimal selection at iteration $t$ \tabularnewline
\midrule 
$A_v^{*}$ & $list$ & a list of $B$ elements taking the (same) optimal selection $c_t^*$ at iteration/segment $v$ \tabularnewline
\midrule 
$g_{t}(c)$ & $[0,1]$ & a gain (reward) occurred at iteration $t$ by pulling an arm $c$\tabularnewline
\midrule 
$\hat{g}_{t}(c)=\frac{g_{t}(c)}{p_{t}^{c}+\gamma}\mutinfo(c=h_{t})$ & $\mathcal{R^+}$ & a normalized gain\tabularnewline
\midrule 
$\alpha=\frac{1}{T},\eta=2\gamma=\sqrt{\frac{2\ln C}{CT}}$ & $\mathcal{R^+}$ & hyperparameters, set by Theorem \ref{thm_regret_TVEXP3M} \tabularnewline
\midrule 
$\gamma=\sqrt{\frac{\ln C}{2CT}}\in[0,1]$ & $\mathcal{R^+}$ & is the exploration parameter\tabularnewline
\midrule 
$G_{T}(c)=\sum_{t=1}^{T}g_{t}(c)$ & $\mathcal{R^+}$ & a total gain if we select an arm $c$ entirely\tabularnewline
\bottomrule
\end{tabular}}
\par\end{centering}

\end{table}

\begin{theorem} (Theorem \ref{thm_regret_TVEXP3M} in the main paper)
Let $T>0$, $C>0$, set $\alpha=\frac{1}{T}$ and $\gamma=\min\left\{ 1,\sqrt{\frac{C\ln(C/B)}{(e-1)BT}}\right\}$, we assume the reward distributions to change at arbitrary time instants, but the total number of change points is no more than $V$ times. The expected regret gained by \textsc{tv.exp3.m} in a batch satisfies the following sublinear regret bound 
\begin{align*}
\mathbb{E} \left[ R_{TB} \right]\le & \left[1+e+V\right] \sqrt{(e-1)\frac{CT}{B}\ln\frac{CT}{B}}.
\end{align*}

\end{theorem}

 
\begin{proof}
We refer to Table \ref{table:notation} for notations. We follow the proof technique presented in \cite{auer2002nonstochastic} and \cite{uchiya2010algorithms} to derive the regret bound. Let $W_{t}=\sum_{c=1}^{C}\omega_{t}^{c}$ and using step 9 in Algorithm \ref{alg:EXP3_MultiplePlay_S}, we have $\frac{\frac{p_{t}^{c}}{B}-\frac{\gamma}{C}}{1-\gamma}=\frac{\omega_{t}^{c}}{W_{t}}$. This equation will be used below.
\small
\begin{align}
\frac{W_{t+1}}{W_{t}} & =\frac{\sum_{c=1}^{C}\omega_{t+1}^{c}}{W_{t}}=\frac{\sum_{c\notin S_{t}(0)}\omega_{t+1}^{c}}{W_{t}}+\frac{\sum_{c\in S_{t}(0)}\omega_{t+1}^{c}}{W_{t}}+\frac{\sum_{c=1}^{C}\frac{e\alpha W_{t}}{C}}{W_{t}}\nonumber \\
 & =\sum_{c\notin S_{t}(0)}\frac{\omega_{t}^{c}}{W_{t}}\times\exp\left(\frac{B\gamma}{C}\hat{g}_{t}(c)\right)+\frac{\sum_{c\in S_{t}(0)}\omega_{t}^{c}}{W_{t}}+e\alpha\nonumber \\
 & \le e\alpha+\sum_{c\notin S_{t}(0)}\frac{\omega_{t}^{c}}{W_{t}}\left[1+\frac{B\gamma}{C}\hat{g}_{t}(c)+(e-2)\left(\frac{B\gamma}{C}\right)^{2}\hat{g}_{t}^{2}(c)\right]+\frac{\sum_{c\in S_{t}(0)}\omega_{t+1}^{c}}{W_{t}}\thinspace\thinspace\thinspace\thinspace\thinspace\thinspace\textrm{by}\thinspace e^{x}\le1+x+(e-2)x^{2}\nonumber \\
 & =e\alpha+\underbrace{\sum_{c\notin S_{t}(0)}\frac{\omega_{t}^{c}}{W_{t}}+\sum_{c\in S_{t}(0)}\frac{\omega_{t+1}^{c}}{W_{t}}}_{1}+\sum_{c\notin S_{t}(0)}\frac{\frac{p_{t}^{c}}{B}-\frac{\gamma}{C}}{1-\gamma}\left[\frac{B\gamma}{C}\hat{g}_{t}(c)+(e-2)\left(\frac{B\gamma}{C}\right)^{2}\hat{g}_{t}^{2}(c)\right]\nonumber \\
 & \le e\alpha+1+\frac{B\gamma}{C\left(1-\gamma\right)}\sum_{c\notin S_{t}(0)}\left(\frac{p_{t}^{c}}{B}-\frac{\gamma}{C}\right)\hat{g}_{t}(c)+\frac{(e-2)}{1-\gamma}\left(\frac{B\gamma}{C}\right)^{2}\sum_{c\notin S_{t}(0)}\left(\frac{p_{t}^{c}}{B}-\frac{\gamma}{C}\right)\hat{g}_{t}^{2}(c)\nonumber \\
 & =e\alpha+1+\frac{\gamma}{C\left(1-\gamma\right)}\sum_{c\notin S_{t}(0)}p_{t}^{c}\hat{g}_{t}(c)-\frac{B\gamma^{2}}{C^{2}\left(1-\gamma\right)}\sum_{c\notin S_{t}(0)}\hat{g}_{t}(c)\nonumber \\
 & +\frac{(e-2)\gamma^{2}B}{\left(1-\gamma\right)C^{2}}\sum_{c\notin S_{t}(0)}p_{t}^{c}\hat{g}_{t}^{2}(c)-\frac{(e-2)\gamma^{3}B^{2}}{\left(1-\gamma\right)C^{2}}\sum_{c\notin S_{t}(0)}\hat{g}_{t}^{2}(c)\nonumber \\
 & \le e\alpha+1+\frac{\gamma}{C\left(1-\gamma\right)}\sum_{c\notin S_{t}(0)}p_{t}^{c}\hat{g}_{t}(c)+\frac{(e-2)\gamma^{2}B}{\left(1-\gamma\right)C^{2}}\sum_{c\notin S_{t}(0)}p_{t}^{c}\hat{g}_{t}^{2}(c)\nonumber \\
 & \le e\alpha+1+\frac{\gamma}{C\left(1-\gamma\right)}\sum_{c\in S_{t}-S_{t}(0)}g_{t}(c)+\frac{(e-2)\gamma^{2}B}{\left(1-\gamma\right)C^{2}}\sum_{c\notin S_{t}(0)}\hat{g}_{t}(c)\thinspace\thinspace\thinspace\thinspace\textrm{by}\thinspace x^{2}\le x,\forall x\in[0,1]\nonumber \\
\frac{W_{t+1}}{W_{t}} & \le e\alpha+1+\frac{\gamma}{C\left(1-\gamma\right)}\sum_{c\in S_{t}-S_{t}(0)}g_{t}(c)+\frac{(e-2)\gamma^{2}B}{\left(1-\gamma\right)C^{2}}\sum_{c\in[C]}\hat{g}_{t}(c)\label{eq:EXP3M_TV_Wt}.
\end{align}
\normalsize
To demonstrate the time-varying property in our bandit problem, we assume the reward distributions change at arbitrary time instants, but the total number of change points is no more than $V$ times. We split the sequence of total decisions $T$ into $V$ segments such that within each segment we have the \textit{time-invariant} reward function. 

We can write $V$ segments as $[T_{1},...T_{2}),[T_{2},....,T_{3}),[T_{V},...,T_{V+1})$
where $T_v$ indicates the starting index of the $v$-th segment, $T_{v+1}$ is the ending index of the $v$-th segment which is also the starting index of the $v+1$-th segment -- using the same notation in \textsc{exp3.s} \cite{auer2002nonstochastic}. Similarly, we can write the optimal sequence $\underbrace{[c^*_{T_1},...c^*_{T_2})}_{=c^*_{T_1}},\underbrace{[c^*_{T_2},....,c^*_{T_3})}_{=c^*_{T_2}},\underbrace{[c^*_{T_V},...,c^*_{T_{V+1}})}_{=c^*_{T_V}}$ where the reward function does not change in each segment, thus takes the same optimal choice $c^*_{T_v}$.


We consider an arbitrary segment $v$ and denote the length $\Delta_{v}=T_{v+1}-T_{v}$. Furthermore, let define the cumulative gain (or reward) achieved by using \textsc{tv.exp3.m} strategy within a segment $v$ that the indices are ranging from $T_v$ to $T_{v+1}-1$
\begin{align*}
G_{\textsc{tv.exp3.m}}(v) & =\sum_{t=T_{v}}^{T_{v+1}-1}\sum_{c\in S_{t}}g_{t}(h_{t}=c).
\end{align*}


Taking ln of Eq. (\ref{eq:EXP3M_TV_Wt}), we get
\begin{align*}
\ln\frac{W_{t+1}}{W_{t}} & \le\ln\left(e\alpha+1+\frac{\gamma/C}{1-\gamma}\sum_{c\in S_{t}-S_{t}(0)}g_{t}(c)+\frac{(e-2)\left(\frac{\gamma}{C}\right)^{2}B}{1-\gamma}\sum_{c=1}^{C}\hat{g}_{t}(c)\right)\\
 & \le e\alpha+\frac{\gamma/C}{1-\gamma}\sum_{c\in S_{t}-S_{t}(0)}g_{t}(c)+\frac{(e-2)\left(\frac{\gamma}{C}\right)^{2}B}{1-\gamma}\sum_{c=1}^{C}\hat{g}_{t}(c) \quad \quad \quad \,\,\,\,\,\,\,\,\,\,\textrm{by}\thinspace1+a\le e^{a}.
\end{align*}
Summing over all indices $t=T_{v},....,T_{v+1}-1$ within a segment $v$-th: 
\begin{align*}
\ln W_{T_{v+1}}-\ln W_{T_{v}} & \le\Delta_{v}e\alpha+\frac{\gamma/C}{1-\gamma}\sum_{t=T_{v}}^{T_{v+1}-1}\sum_{c\in S_{t}-S_{t}(0)}g_{t}(c)+\frac{(e-2)\left(\frac{\gamma}{C}\right)^{2}B}{1-\gamma}\sum_{t=T_{v}}^{T_{v+1}-1}\sum_{c=1}^{C}\hat{g}_{t}(c)\\
 & \le\Delta_{v}e\alpha+\frac{\gamma/C}{1-\gamma}\sum_{t=T_{v}}^{T_{v+1}-1}\sum_{c\in S_{t}-S_{t}(0)}g_{t}(c)+\frac{(e-2)\left(\frac{\gamma}{C}\right)^{2}B}{1-\gamma}\sum_{t=T_{v}}^{T_{v+1}-1}\sum_{c=1}^{C}\hat{g}_{t}(c).
\end{align*}
Let $j=c_{T_{v}}^{*}=...=c_{T_{v+1}-1}^{*}$ be the optimal choice in the sequence $v$, we have $\omega_{j}^{\left(T_{v+1}\right)}=\omega_{j}^{(T_{v+1}-1)}\times\exp\left(\frac{B\gamma}{C}\hat{g}_{t}(j)\right)+\frac{e\alpha}{C}W_{T_{v+1}-1},\forall j\notin S_{T_{v+1}-1}(0)$
and $\omega_{j}^{\left(T_{v+1}\right)}=\omega_{j}^{(T_{v+1}-1)},\forall j\in S_{T_{v+1}-1}(0)$
\begin{align*}
\omega_{j}^{\left(T_{v+1}\right)} & =\omega_{j}^{(T_{v+1}-1)}\times\exp\left(\frac{B\gamma}{C}\hat{g}_{t}(j)\right)+\frac{e\alpha}{C}W_{T_{v+1}-1}\\
 & \ge\omega_{j}^{(T_{v+1}-1)}\times\exp\left(\frac{B\gamma}{C}\hat{g}_{t}(j)\right)\\
 & \ge\omega_{j}^{(T_{v+1}-2)}\times\exp\left(\frac{B\gamma}{C}\hat{g}_{t}(j)\right)\exp\left(\frac{B\gamma}{C}\hat{g}_{\left(T_{v+1}-2\right)}(j)\right)\\
 & \ge\omega_{j}^{(T_{v})}\times\exp\left(\frac{B\gamma}{C}\sum_{t=T_{v}\mid t:j\notin S_{t}(0)}^{T_{v+1}-1}\hat{g}_{t}(j)\right)\\
 & \ge\frac{e\alpha}{C}W_{T_{v}}\times\exp\left(\frac{B\gamma}{C}\sum_{t=T_{v}\mid t:j\notin S_{t}(0)}^{T_{v+1}-1}\hat{g}_{t}(j)\right)\\
 & \ge\frac{\alpha}{C}W_{T_{v}}\times\exp\left(\frac{B\gamma}{C}\sum_{t=T_{v}\mid t:j\notin S_{t}(0)}^{T_{v+1}-1}\hat{g}_{t}(j)\right).
\end{align*}
We now consider the lower bound by taking the optimal (batch) set $A_{v}^{*}\subset[C]$
for $B$ elements with the maximum total of reward within the segment
$v$: $\sum_{c\in A_{v}^{*}}\sum_{t=1}^{T}g_{t}(c)$. Since we have
the fact that $\sum_{j\in A_{v}^{*}}\omega_{j}^{\left(T_{v+1}\right)}\ge B\left(\prod_{j\in A_{v}^{*}}\omega_{j}^{\left(T_{v+1}\right)}\right)^{1/B}$
by Cauchy Swatchz inequality, we continue the above formula
\begin{align*}
\ln\left(W_{T_{v+1}-1}\right)-\ln W_{T_{v}}\ge & \ln B+\frac{1}{B}\sum_{j\in A_{v}^{*}}\ln\omega_{j}^{\left(T_{v+1}\right)}-\ln W_{T_{v}}\\
= & \ln B+\frac{1}{B}\sum_{j\in A_{v}^{*}}\ln\frac{\alpha}{C}W_{T_{v}}+\frac{1}{B}\sum_{j\in A_{v}^{*}}\frac{B\gamma}{C}\sum_{t=T_{v}\mid t:j\notin S_{t}(0)}^{T_{v+1}-1}\hat{g}_{t}(j)-\ln W_{T_{v}}\\
= & \ln B+\frac{1}{B}\sum_{j\in A_{v}^{*}}\ln\frac{\alpha}{C}+\sum_{j\in A_{v}^{*}}\frac{\gamma}{C}\sum_{t=T_{v}\mid t:j\notin S_{t}(0)}^{T_{v+1}-1}\hat{g}_{t}(j)\\
= & \ln\frac{B\alpha}{C}+\sum_{j\in A_{v}^{*}}\frac{\gamma}{C}\sum_{t=T_{v}\mid t:j\notin S_{t}(0)}^{T_{v+1}-1}\hat{g}_{t}(j).
\end{align*}
Combining both the lower bound and upper bound, we get
\begin{align*}
\ln\frac{B\alpha}{C}+\sum_{j\in A_{v}^{*}}\frac{\gamma}{C}\sum_{t=T_{v}\mid t:j\notin S_{t}(0)}^{T_{v+1}-1}\hat{g}_{t}(j) & \le\Delta_{v}e\alpha+\frac{\gamma/C}{1-\gamma}\sum_{t=T_{v}}^{T_{v+1}-1}\sum_{c\in S_{t}-S_{t}(0)}g_{t}(c) \nonumber \\ +\frac{(e-2)\left(\frac{\gamma}{C}\right)^{2}B}{1-\gamma}\sum_{t=T_{v}}^{T_{v+1}-1}\sum_{c=1}^{C}\hat{g}_{t}(c).
\end{align*}
Summing over all segments $v=1....V$
\small
\begin{align}
V\ln\frac{B\alpha}{C}+\sum_{v=1}^{V}\sum_{j\in A_{v}^{*}}\frac{\gamma}{C}\sum_{t=T_{v}\mid t:j\notin S_{t}(0)}^{T_{v+1}-1}\hat{g}_{t}(j)\le & Te\alpha+\frac{\gamma/C}{1-\gamma}\sum_{v=1}^{V}\sum_{t=T_{v}}^{T_{v+1}-1}\sum_{c\in S_{t}-S_{t}(0)}g_{t}(c) \nonumber \\
 & +\frac{(e-2)\left(\frac{\gamma}{C}\right)^{2}B}{1-\gamma}\sum_{v=1}^{V}\sum_{t=T_{v}}^{T_{v+1}-1}\sum_{c=1}^{C}\hat{g}_{t}(c) \nonumber \\
\frac{VC}{\gamma}\ln\frac{B\alpha}{C}+\sum_{v=1}^{V}\sum_{j\in A_{v}^{*}}\sum_{t=T_{v}\mid t:j\notin S_{t}(0)}^{T_{v+1}-1}\hat{g}_{t}(j)\le & \frac{C}{\gamma}Te\alpha+\frac{1}{1-\gamma}\sum_{v=1}^{V}\sum_{t=T_{v}}^{T_{v+1}-1}\sum_{c\in S_{t}-S_{t}(0)}g_{t}(c) \nonumber \\
 & +\frac{(e-2)\left(\frac{\gamma B}{C}\right)}{1-\gamma}\sum_{v=1}^{V}\sum_{t=T_{v}}^{T_{v+1}-1}\sum_{c=1}^{C}\hat{g}_{t}(c) \nonumber \\
\sum_{v=1}^{V}\sum_{t=T_{v}\mid j\in S_{t}(0)}^{T_{v+1}-1}\sum_{j\in A_{v}^{*}}g_{t}(j)+\sum_{v=1}^{V}\sum_{j\in A_{v}^{*}}\sum_{t=T_{v}\mid j\notin S_{t}(0)}^{T_{v+1}-1}\hat{g}_{t}(j)\le & \frac{C}{\gamma}Te\alpha+\frac{1}{1-\gamma}\sum_{v=1}^{V}\sum_{t=T_{v}}^{T_{v+1}-1}\sum_{c\in S_{t}(0)}g_{t}(c) \nonumber \\
& + \frac{1}{1-\gamma}\sum_{v=1}^{V}\sum_{t=T_{v}}^{T_{v+1}-1}\sum_{c\in S_{t}-S_{t}(0)}g_{t}(c) \nonumber \\
 & +\frac{(e-2)\left(\frac{\gamma B}{C}\right)}{1-\gamma}\sum_{v=1}^{V}\sum_{t=T_{v}}^{T_{v+1}-1}\sum_{c=1}^{C}\hat{g}_{t}(c) \nonumber \\
 & -\frac{VC}{\gamma}\ln\frac{B\alpha}{C} \label{eq:tvexp3m_GTB}
 \end{align}
 \normalsize
 Let us denote the optimal gain over all iterations and all parallel agents  $G^*_{TB} = \sum_{v=1}^{V}\sum_{t=T_{v}}^{T_{v+1}-1}\sum_{j\in A_{v}^{*}}g_{t}(j)$. We also denote the cumulative gain achieved by using \textsc{tv.exp3.m} algorithms $G_{\textsc{tv.exp3.m}} = \sum_{v=1}^{V}\sum_{t=T_{v}}^{T_{v+1}-1}\sum_{c \in S_t} g_t(c).$  Then, we continue Eq. (\ref{eq:tvexp3m_GTB}) as
\begin{equation} \label{eq:tvexp3m_G_TB}
    G^*_{TB}\le \frac{C}{\gamma}Te\alpha+\frac{1}{1-\gamma}G_{\textsc{tv.exp3.m}}  +\frac{(e-2)\left(\frac{\gamma B}{C}\right)}{1-\gamma}\sum_{v=1}^{V}\sum_{t=T_{v}}^{T_{v+1}-1}\sum_{c=1}^{C}\hat{g}_{t}(c)-\frac{VC}{\gamma}\ln\frac{B\alpha}{C}   
\end{equation}  
where the number of element in a batch $|S_t| = |A^*_j| = B$. Let take expectation both sides of Eq. (\ref{eq:tvexp3m_G_TB}), we have $\mathbb{E}\left[\hat{g}_{t}(c)\mid S_1,S_2,...,S_{i-1}\right]=g_{t}(c)$
from the fact that DepRound \cite{gandhi2006dependent} selects action $c$ with probability
$p_{c}(t)$, we obtain
\begin{align*}
G^*_{TB} & \le\frac{1}{\left(1-\gamma\right)}\mathbb{E}\left[G_{\textsc{tv.exp3.m}}\right]+\frac{(e-2)\left(\frac{\gamma B}{C}\right)}{1-\gamma}\sum_{v=1}^{V}\sum_{t=T_{v}}^{T_{v+1}-1}\sum_{c=1}^{C}g_{t}(c)+\frac{C}{\gamma}Te\alpha-\frac{VC}{\gamma}\ln\frac{B\alpha}{C}.
\end{align*}

\begin{algorithm*}[t]
\caption{\textsc{tv.exp3.m} algorithm\label{alg:EXP3_MultiplePlay_S}}
\label{alg:tv_exp3_full}

\begin{algor}
\item [{{*}}] Input: $\gamma=\sqrt{\frac{C\ln(C/B)}{(e-1)BT}}$, $\alpha=\frac{1}{T},$$C$ \#categorical
choice, $T$ \#max iteration, $B$ \#multiple play
\end{algor}
\begin{algor}[1]
\item [{{*}}] Init $\omega_{c}=1,\forall c=1...C$ and denote $\eta=(\frac{1}{B}-\frac{\gamma}{C})\frac{1}{1-\gamma}$
\item [{for}] $t=1$ to $T$
\item [{if}] $\arg\max_{c\in[C]}\omega_{c}\ge\eta\sum_{c=1}^{C}\omega(c)$
\item [{{*}}] $\nu$ s.t. $\frac{\nu}{\eta}=\sum_{\omega_{t}(c)\ge \nu} \nu+\sum_{\omega_{t}(c)<\omega_{t}(c)}\omega_{t}(c)$
\item [{{*}}] Set $S_0=( c:\omega_{t}(c)\ge \nu )$
and $\omega_{t}(S_0)=\nu$ 
\item [{else}]~
\item [{{*}}] Set $S_0=\emptyset$
\item [{endif}]~
\item [{{*}}] Compute $p_{t}^{c}=B\left((1-\gamma)\frac{\omega_{c}}{\sum_{c=1}^{C}\omega_{c}}+\frac{\gamma}{C}\right),\forall c$
\item [{{*}}] $S_{t}=\textrm{DepRound}\left(B,\left[p_{t}^{1}p_{t}^{2}....p_{t}^{C}\right]\right)$
\item [{{*}}] Observe the reward $g_{t}(c)=f(h_{t}=c)$ for $c\in S_{t}$
\item [{{*}}] $\hat{g}_{t}(c)=\frac{g_{t}(c)}{p_{t}^{c}},\forall c\in S_{t}$
and $\hat{g}_{t}(c)=0$ otherwise
\item [{{*}}] $\forall c\notin S_0:$ update $\omega_{c}=\omega_{c}\times\exp\left(B\gamma\hat{g}_{t}(c)/C\right)+\frac{e\alpha}{C}\sum_{i=1}^{C}\omega_{c},$
\item [{{*}}] $\forall c=S_0: $ update  $\omega_{c}=\omega_{c}+\frac{e\alpha}{C}\sum_{i=1}^{C}\omega_{c}$
\item [{endfor}]~
\end{algor}
\begin{algor}
\item [{{*}}] Output: $\mathcal{D}_{T}$
\end{algor}
\end{algorithm*}

We finally have
\small
\begin{align}
(1-\gamma)G^*_{TB}\le & \mathbb{E}\left[G_{\textsc{tv.exp3.m}}\right]+(e-2)\left(\frac{\gamma B}{C}\right)\frac{C}{B}G^*_{TB}+\frac{(1-\gamma)}{\gamma}CTe\alpha-\frac{(1-\gamma)VC}{\gamma}\ln\frac{B\alpha}{C} \label{eq:EXP3MS_Gmax} \\
G^*_{TB}-\mathbb{E}\left[G_{\textsc{tv.exp3.m}}\right]\le & (e-1)\gamma G^*_{TB}+\frac{(1-\gamma)}{\gamma}CTe\alpha+\frac{(1-\gamma)VC}{\gamma}\ln\frac{C}{B\alpha}\nonumber \\
\le & (e-1)\gamma TB+\frac{(1-\gamma)}{\gamma}CTe\alpha+\frac{VC}{\gamma}\ln\frac{C}{B\alpha}-VC\ln\frac{C}{B\alpha}\nonumber \quad \quad \text{by } G^*_{TB}\le TB \\
\le & (e-1)\gamma TB+\frac{1}{\gamma}CTe\alpha+\frac{VC}{\gamma}\ln\frac{C}{B\alpha}.
\end{align}
\normalsize
In Eq. (\ref{eq:EXP3MS_Gmax}), we use the fact that $\sum_{v=1}^{V}\sum_{t=T_{v}}^{T_{v+1}-1}\sum_{c=1}^{C}g_{t}(c)\le\frac{C}{B}G^*_{TB}=\frac{C}{B}\sum_{v=1}^{V}\sum_{t=T_{v}}^{T_{v+1}-1}\sum_{c\in A^{*}}g_{t}(c)$.

Set $\gamma=\min\left\{ 1,\sqrt{\frac{C\ln(C/B)}{(e-1)BT}}\right\} $
and $\alpha=\frac{1}{T}$. Then, we have the cumulative regret over all iterations and considering the best element in a batch
\begin{align}
 \mathbb{E}\left[ R_{TB} \right]&=  \mathbb{E} \left[ \sum_{v=1}^{V}\sum_{t=T_{v}}^{T_{v+1}-1}  \overbrace{ \max_{ \forall c \le C}g_{t}(c) }^{\text{optimal arms}}\right] - \mathbb{E} \left[  \sum_{v=1}^{V}\sum_{t=T_{v}}^{T_{v+1}-1} \overbrace{ \max_{ \forall c \in S_t} g_t(c) }^{\text{best arm in a batch $S_t$}} \right]  \label{eq:R_T_1} \\
 & \le \frac{1}{B} G^*_{TB} - \frac{1}{B} \mathbb{E}\left[G_{\textsc{tv.exp3.m}}\right]  \label{eq:R_T_2} \\
 & \le \frac{1}{B}\sqrt{\frac{C (e-1)\ln(C/B)}{BT}}T+\frac{1}{B} \frac{\sqrt{C(e-1)BT}}{\sqrt{\ln(C/B)}}e+\frac{1}{B} \frac{VC}{\sqrt{\frac{C\ln(C/B)}{(e-1)BT}}}\ln\frac{CT}{B}\\
 & \le  \left[1+e+V\right] \sqrt{(e-1)\frac{CT}{B}\ln\frac{CT}{B}} \label{eq:R_T_3}
\end{align}

where we obtain Eq. (\ref{eq:R_T_2}) because the best gain should be greater than the average gain of a batch $ \sum_{v=1}^{V}\sum_{t=T_{v}}^{T_{v+1}-1}\max_{ \forall c \in S_t} g_t(c) \ge \frac{1}{B} \sum_{v=1}^{V}\sum_{t=T_{v}}^{T_{v+1}-1}\sum_{c \in S_t} g_t(c)$. 

Given the number of changing points in our reward function is bounded $V \ll \sqrt{T}$ and the number of category $C$ is a constant, our regret bound achieves sublinear regret rate with the number of iterations $T$, i.e., $\lim_{T \rightarrow\infty}\frac{ \mathbb{E}\left[ R_{TB} \right] }{T}=0$. 

\end{proof}

In the above derivation, $T$ refers to the number of bandit updates which is equivalent to the number of $t_{\mathrm{ready}}$  in PB2 setting \cite{pb2}.

\subsection{Theoretical Derivations for PB2-Mult} \label{appendix_subsec:Theory_PB2_D}

We adapt Lemma 1 in \cite{pb2} to handle categorical variables. We first restate some notations used in \cite{pb2}. Let $F_t^{c_t}(x_t)$ be an objective function under a given set of continuous hyperparameters $x_t$ and categorical variable $c_t$ at timestep $t$. An example of $F_t^{c_t}(x_t)$ could be the reward for a deep RL agent. When training for a total of $T$ steps, our goal is to maximize the final performance $F_T^{c_T}(x_T)$. We formulate this problem as optimizing the time-varying black-box reward function $f_t$, over $\mathcal{D}$.

\begin{lemma} \label{lem:maxreward_minregret}
Maximizing the final performance $F_T$ of a model with respect to a given continuous hyperparameter schedule  $\{x_t\}_{t=1}^T$ and a categorical hyperparameter schedule $\{c_t\}_{t=1}^T$ is equivalent to maximizing the time-varying black-box function $f_t(x_t)$ and minimizing the corresponding cumulative regret $r_t(x_t)$,
\begin{align}
\max F^{c_T}_T(x_T)= \max \sum^T_{t=1} f_{c_t}(x_t) =  \min \sum^T_{t=1} r_t(x_t).
\end{align}
\end{lemma}

\begin{proof}

At each $t_{\mathrm{ready}}$, we select a categorical variable $c_t \in \{1,...,C\}$ and a continuous variable $x_t \in \mathcal{D}^{c_t}$. We would emphasize that the continuous variables are conditioned on the choice of $c_t$. We observe and record noisy observations, $y_t = f_{c_t}(x_t) + \epsilon_t$, where $\epsilon_t \sim \mathcal{N}(0, \sigma^2\mathbf{I})$ for some fixed $\sigma^2$. The function $f_t$ represents the change in $F_t$ after training for $t_{\mathrm{ready}}$ steps, i.e. $f_{c_t}(x_t)=F^{c_t}_t(x_t) - F_{t-t_\mathrm{ready}}^{c_{t-t_\mathrm{ready}}}(x_{t-t_\mathrm{ready}})$. 
We define the best choice at each timestep as $x_t^* = \arg \max_{x_t\in\mathcal{D}^{c_t}}f_{c_t}(x_t)$, and $c_t^* = \arg \max_{c_t, x_t}f_{c_t}(x_t)$. The intermediate regret of each decision is defined as $r_t = f_{c^*_t}(x_t^*) - f_{c_t}(x_t)$ where $f_{c^*_t}(x_t^*)$ is an unknown constant.

We have a reward at the starting iteration $F^{c_1}_1(x_1)$ as a constant that allows us to write the objective function as:
\small
\begin{align}
  F^{c_T}_{T}(x_{T})-F^{c_1}_{1}(x_{1})= & \underbrace{F_{T}^{c_{T}}(x_{T})-F_{T-1}^{c_{T-1}}(x_{T-1})}_{f_{c_{T-1}}(x_{T-1})}+...+ \underbrace{F_{3}^{c_{3}}(x_{3})-F_{2}^{c_{2}}(x_{2})}_{f_{c_2}(x_2)}+ \underbrace{ F_{2}^{c_{2}}(x_{2})-F_{1}^{c_{1}}(x_{1})}_{f_{c_1}(x_1)}.
    \label{eq:equivalent_reward_regret}
\end{align}
\normalsize
Therefore, maximizing the left of Eq. (\ref{eq:equivalent_reward_regret}) is equivalent to minimizing the cummulative regret as follows:
\small
\begin{align}
    \max \left[ F^{c_T}_T(x_T) - F^{c_1}_1(x_1) \right] &=  \max \sum_{t=1}^T F^{c_t}_t(x_t) - F^{c_{t-1}}_{t-1}(x_{t-1}) = \max \sum_{t=1}^{T-1} f_{c_t}(x_t) = \min \sum_{t=1}^{T-1} r_t(x_t)\nonumber
\end{align}
\normalsize
where we define $f_{c_t}(x_t)=F^{c_t}_t(x_t)-F^{c_{t-1}}_{t-1}(x_{t-1})$, the regret $r_t=f_{c^*_t}(x^*_t)-f_{c_t}(x_t)$.
\end{proof}

We then restate Theorem 2 in \cite{pb2} which will be used in deriving the convergence guarantee for PB2-Mult.

\begin{theorem} (Theorem 2 in \cite{pb2})
Let the domain $\mathcal{D}\subset[0,r]^{d}$ be compact and convex where $d$ is the dimension and suppose that the kernel is such that $f_t \sim GP(0,k)$ is almost surely continuously differentiable and satisfies Lipschitz assumptions $\forall L_t \ge0,t\le\mathcal{T}, \forall j\le d, p(\sup\left|\frac{\partial f_{t}(\bx)}{\partial\bx^{(j)}}\right|\ge L_t)\le ae^{-\left(L_t/b\right)^{2}}$ for some $a,b$. Pick $\delta\in(0,1)$, set $\beta_{T}=2\log\frac{\pi^{2}T^{2}}{2\delta}+2d\log rdbT^{2}\sqrt{\log\frac{da\pi^{2}T^{2}}{2\delta}}$ and define $C_{1}=32/\log(1+\sigma_{f}^{2})$, the PB2 algorithm satisfies the following regret bound after $T$ time steps over $B$ parallel agents with probability at least $1-\delta$:
\begin{align*}
R_{TB}=\sum_{t=1}^{T}f_{t}(\bx_{t}^{*})-\max_{b \le B} f_{t}(\bx_{t,b})\le \sqrt{C_{1}T\beta_{T}\left(\frac{T}{\tilde{N}B}+1\right)\left(\gamma_{\tilde{N}B}+\left[\tilde{N}B\right]^{3}\omega\right)}+2
\end{align*}
the bound holds for any block length
$\tilde{N}\in\left\{ 1,...,T\right\} $ and $B\ll T$.
\end{theorem}


Next, we are going to derive the main theorem of the paper.
\begin{theorem} (Theorem \ref{thm:PB2-D} in the main paper) \label{thm_appendix:PB2-D}
Let the domain for continuous variables $\mathcal{D}\subset[0,r]^{d}$ be compact and convex where $d$ is the dimension and suppose that the kernel is such that $f_t \sim GP(0,k)$ is almost surely continuously differentiable and satisfies Lipschitz assumptions $\forall L_t \ge0,t\le\mathcal{T}, \forall j\le d ,p(\sup\left|\frac{\partial f_{t}(\bx)}{\partial\bx^{(j)}}\right|\ge L_t)\le ae^{-\left(L_t/b\right)^{2}}$ for some $a,b$. 

Assume the reward distributions to change at arbitrary time instants, but
the total number of change points is no more than $V$. Set $\alpha=\frac{1}{T}$, $\gamma=\min\left\{ 1,\sqrt{\frac{C\ln(C/B)}{(e-1)BT}}\right\}$, $\beta_{T}=2\log\frac{\pi^{2}T^{2}}{2\delta}+2d\log rdbT^{2}\sqrt{\log\frac{da\pi^{2}T^{2}}{2\delta}}$, pick $\delta\in(0,1)$ and define $C_{1}=32/\log(1+\sigma_{f}^{2})$, the PB2-Mult algorithm satisfies the following regret bound after $T$ time steps over $B$ parallel agents with probability at least $1-\delta$:
\small
\begin{align*}
\mathbb{E}\left[  R_{TB} \right] & \le \left[1+e+V\right] \sqrt{(e-1)\frac{CT}{B}\ln\frac{CT}{B}} + \sqrt{C_{1} T_{c^*_t} \beta_{T_{c^*_t}}\left(\frac{T_{c^*_t}}{\tilde{N}B}+1\right)\left(\gamma_{\tilde{N}B}+\left[\tilde{N}B\right]^{3}\omega\right)}+2.
\end{align*}
\normalsize
The bound holds for any $\tilde{N} \in \{1,...,T_{c^*_t}\}$ and $V \ll \sqrt{T}$.
\end{theorem}

\begin{proof}

We expand the cumulative regret and optimize it using time-varying GP bandit optimization \cite{bogunovic2016time,pb2}
\begin{align*}
R_{TB} & =  \sum_{t=1}^{T}  f^{*} - \max_{b=1...B} \sum_{t=1}^{T} f_{c_{t,b}}(x_{t,b}) \\
& = \max_{b=1,...,B} \sum_{t=1}^{T}  f_{c_{t,b}^{*}}(x_{t,b})- \max_{b=1,...,B} \sum_{t=1}^{T} f_{c_{t,b}}(x_{t,b}) + \sum_{t=1}^{T}  f^{*}- \max_{b=1,...,B} \sum_{t=1}^{T} f_{c_{t,b}^{*}}(x_{t,b}) \\
\end{align*}
where $b\in \{1,...,B\}$ is an agent's index being trained in parallel, $f_{c_{t}^{*}}(x_{t})=\max_{\forall c_{t}\in \{1,...,C\}}f_{c_{t}}(x_{t})$ and $c_{t}^{*}=\arg\max_{\forall c\in\{1,...C\}}f_{c}(x_{t})$ is
the optimal categorical choice at iteration $t$.


We bound the two terms separately as follows. The first term is bounded by Theorem \ref{thm_regret_TVEXP3M}. We assume the process of generating the arm $c_{t}$'s reward $f_{c_{t}}(.):=f_{c_{t}}(x_{t})$ is by the ``adversary'' that \textsc{tv.exp3.m} does not have the direct control on the selection of $x_{t}$. Particularly, $x_{t}$ will be chosen by \textsc{tv-gp-bucb} (as part of PB2) in Eqn. (\ref{eq:gp_bucb_acq}). We take the expectation of the first term to have
\small
\begin{align*}
\mathbb{E}\left[ \max_{b=1...B} \sum_{t=1}^{T} \underbrace{f_{c_{t,b}^{*}}(.)}_\text{pull\,optimal\,arm} \right]- \mathbb{E}\left[ \max_{b=1...B} \sum_{t=1}^{T}  \underbrace{f_{c_{t,b}}(.)}_\text{pull\,arm\,$c_{t}$} \right] &= \mathbb{E}\left[ \sum_{t=1}^{T} \underbrace{f_{c_{t}^{*}}(.)}_\text{pull\,optimal\,arm} \right]- \mathbb{E}\left[ \max_{b=1...B} \sum_{t=1}^{T}  \underbrace{f_{c_{t,b}}(.)}_\text{pull\,arm\,$c_{t}$} \right] \\
& = \mathbb{E}\left[ \tilde{R}_{TB}\right] \\
 & \le \left[1+e+V\right] \sqrt{(e-1)\frac{CT}{B}\ln\frac{CT}{B}}
\end{align*}
\normalsize
where $R_{TB}$ is cumulative regret of the \textsc{tv.exp3.m}, defined in Theorem \ref{thm_regret_TVEXP3M}.

Assuming the best arm (the best categorical choice) $c^*_t$ can be identified by \textsc{tv.exp3.m} in Theorem \ref{thm_regret_TVEXP3M}. The second term is the regret bound presented in Theorem 2 of the PB2 \cite{pb2}:
\begin{align*}
\sum_{t=1}^{T} f_t^{*}- \max_{b=1...B} \sum_{t=1}^{T} f_{c_{t,b}^{*}}(x_{t,b}) = & \sum_{t=1}^{T} f_t^{*}- \max_{b=1...B} \sum_{t=1}^{T} f_{c_{t}^{*}}(x_{t,b}) \\
&  \le \underbrace{ \sqrt{C_{1} T_{c^*_t} \beta_{T_{c^*_t}}\left(\frac{T_{c^*_t}}{\tilde{N}B}+1\right)\left(\gamma_{\tilde{N}B}+\left[\tilde{N}B\right]^{3}\omega\right)}+2}_{\mathcal{O}(PB2)}
\end{align*}
where $T_{c^*_t}$ denotes the number of times the optimal category $c^*_t$  (or optimal arm) is selected.
We follow \cite{bogunovic2016time,pb2} to denote a block length $\Tilde{N}\in \{1,....T_{c^*_t}\}$ in which the function does not change significantly. In the above equation, $\omega \in [0,1]$ is a time-varying hyperparameter which is estimated by maximizing the GP log marginal likelihood, $B$ is a batch size or the number of population agents, $\gamma_{\tilde{N}B}$ is the maximum information gain \cite{gp_ucb} defined within a block $\Tilde{N}$ over $B$ parallel agents. 

Note that the theoretical result for PB2 comes with the additional smoothness assumption on the kernel $k$  that holds for some $\left(a,b\right)$ and
$\forall L_t\ge0$. The kernel satisfies for all dimensions $j=1,...,d$
\begin{align}
\forall L_t & \ge0,t\le\mathcal{T},p(\sup\left|\frac{\partial f_{t}(\bx)}{\partial\bx^{(j)}}\right|\ge L_t)\le ae^{-\left(L_t/b\right)^{2}}.
\end{align}
We combine the two terms to obtain the final regret bound of the PB2-Mult algorithm:
\small
\begin{align*}
\mathbb{E}\left[  R_{TB} \right] & \le \left[1+e+V\right] \sqrt{(e-1)\frac{CT}{B}\ln\frac{CT}{B}} + \sqrt{C_{1} T_{c^*_t} \beta_{T_{c^*_t}}\left(\frac{T_{c^*_t}}{\tilde{N}B}+1\right)\left(\gamma_{\tilde{N}B}+\left[\tilde{N}B\right]^{3}\omega\right)}+2.
\end{align*}
\normalsize
The final regret bound is a summation of two sub-linear terms. Therefore, the regret bound grows sublinearly with the number of iterations $T$, i.e., $\lim_{T \rightarrow\infty}\frac{ \mathbb{E}\left[ R_{TB} \right] }{T}=0$. This is under a common assumption \cite{bogunovic2016time,pb2} that the time-varying function is correlated, i.e., we have  $\omega \rightarrow 0$, thus $\tilde{N} \rightarrow T$. We also note that when $T \rightarrow \infty$, then $T_{c^*_T} \rightarrow\infty$ due to Theorem \ref{thm_regret_TVEXP3M}.
\end{proof}


\subsection{Gradients}
\label{sec:kernel_gradients}

We optimize the GP hyperparameters by maximizing the log marginal likelihood \cite{RasmussenGP}. We fit the GP hyperparameters by maximizing their posterior probability (MAP), $p\left(\sigma_{x},\sigma_{t}\mid\bX,\mathbf{t},\by\right)\propto p\left(\sigma_{x},\sigma_{t},\bX,\mathbf{t},\by\right)$, which, thanks to the Gaussian likelihood, is available in closed form  
\begin{align}
\mathcal{L}:=\ln p\left(\by,\bX,\mathbf{t},\theta\right)= & \frac{1}{2}\by^{T}\left(\bK+\sigma_{y}^{2}\idenmat_{N}\right)^{-1}\by-\frac{1}{2}\ln\left|\bK+\sigma_{y}^{2}\idenmat_{N}\right|+\ln p_{hyp}\left(\theta\right)+\textrm{const}\label{eq:MarginalLLK}
\end{align}

 where $\idenmat_{N}$ is the identity matrix in dimension $N$ (the number of points in the training set), and $p_{hyp}(\theta)$ is the prior over hyperparameters. We optimize Eq. (\ref{eq:MarginalLLK}) with a gradient-based optimizer, providing the analytical gradient to the algorithm.

Our time-varying CoCaBO kernel is defined as
\begin{align*}
k_{z}(z,z') & =(1-\lambda)\left(k_{xt}+k_{ht}\right)+\lambda k_{xt}k_{ht}
\end{align*}
where $k_{xt}=k_{\textrm{Continuous}}(x,x')\times k_{\textrm{time1}}(t,t')$,
$k_{ht}=k_{\textrm{Categorical}}(h,h')\times k_{\textrm{time2}}(t,t')$,
$k_{\textrm{Continuous}}(x,x')=\sigma_{1}\times\exp\left(-\frac{||x-x'||^{2}}{l}\right)$,
$k_{\textrm{Categorical}}(h,h')=\frac{\sigma_{2}}{C}\sum\mutinfo(h,h')$,
$k_{\textrm{time1}}(t,t')=(1-\epsilon_{1})^{\frac{|t-t'|}{2}}$ and
$k_{\textrm{time2}}(t,t')=(1-\epsilon_{2})^{\frac{|t-t'|}{2}}$.

Our hyperparameters include $\theta=\left\{ \epsilon_{1},\epsilon_{2},l,\sigma_{1},\sigma_{2},\lambda\right\} $.
We need to compute the gradients of $\frac{\partial\mathcal{L}}{\partial\epsilon_{1}},\frac{\partial\mathcal{L}}{\partial\epsilon_{2}},\frac{\partial\mathcal{L}}{\partial l},\frac{\partial\mathcal{L}}{\partial\sigma_{1}},\frac{\partial\mathcal{L}}{\partial\sigma_{2}},\frac{\partial\mathcal{L}}{\partial\lambda}$
as follows:

\begin{itemize}
\item The gradient of $\frac{\partial\mathcal{L}}{\partial\epsilon_{1}}$
\begin{align*}
\frac{\partial\mathcal{L}}{\partial\epsilon_{1}} & =\frac{\partial\mathcal{L}}{\partial k_{z}}\times\frac{\partial k_{z}}{\partial\epsilon_{1}}\\
\frac{\partial k_{z}}{\partial\epsilon_{1}} & =\left(1-\lambda\right)k_{\textrm{Continuous}}(x,x')\frac{\partial k_{\textrm{time1}}}{\partial\epsilon_{1}}+\lambda k_{ht}k_{\textrm{Continuous}}(x,x')\frac{\partial k_{\textrm{time1}}}{\partial\epsilon_{1}}\\
\frac{\partial k_{\textrm{time1}}}{\partial\epsilon_{1}} & =-\frac{|t-t'|}{2}\left(1-\epsilon_{1}\right)^{\frac{|t-t'|}{2}-1}
\end{align*}
\item The gradient of $\frac{\partial\mathcal{L}}{\partial\epsilon_{2}}$
\begin{align*}
\frac{\partial\mathcal{L}}{\partial\epsilon_{2}} & =\frac{\partial\mathcal{L}}{\partial k_{z}}\times\frac{\partial k_{z}}{\partial\epsilon_{2}}\\
\frac{\partial k_{z}}{\partial\epsilon_{2}} & =\left(1-\lambda\right)k_{\textrm{Categorical}}(h,h')\frac{\partial k_{\textrm{time2}}}{\partial\epsilon_{2}}+\lambda k_{xt}k_{\textrm{Categorical}}(h,h')\frac{\partial k_{\textrm{time2}}}{\partial\epsilon_{2}}\\
\frac{\partial k_{\textrm{time2}}}{\partial\epsilon_{2}} & =-\frac{|t-t'|}{2}\left(1-\epsilon_{2}\right)^{\frac{|t-t'|}{2}-1}
\end{align*}
\item The gradient of $\frac{\partial\mathcal{L}}{\partial l}$
\begin{align*}
\frac{\partial\mathcal{L}}{\partial l}= & \frac{\partial\mathcal{L}}{\partial k_{z}}\times\frac{\partial k_{z}}{\partial l}\\
\frac{\partial k_{z}}{\partial l}= & \left(1-\lambda\right)k_{\textrm{time1}}(t,t')\frac{\partial k_{\textrm{Continuous}}}{\partial l}+\lambda k_{ht}k_{\textrm{time1}}(t,t')\frac{\partial k_{\textrm{Continuous}}}{\partial l}\\
\frac{\partial k_{\textrm{Continuous}}}{\partial l}= & \frac{||x-x'||^{2}}{l^{2}}k_{\textrm{Continuous}}(x,x')
\end{align*}
\item The gradient of $\frac{\partial\mathcal{L}}{\partial\sigma_{1}}$
\begin{align*}
\frac{\partial\mathcal{L}}{\partial\sigma_{1}}= & \frac{\partial\mathcal{L}}{\partial k_{z}}\times\frac{\partial k_{z}}{\partial\sigma_{1}}\\
\frac{\partial k_{z}}{\partial\sigma_{1}}= & \left(1-\lambda\right)k_{\textrm{time1}}(t,t')\frac{\partial k_{\textrm{Continuous}}}{\partial\sigma_{1}}+\lambda k_{ht}k_{\textrm{time1}}(t,t')\frac{\partial k_{\textrm{Continuous}}}{\partial\sigma_{1}}\\
\frac{\partial k_{\textrm{Continuous}}}{\partial\sigma_{1}}= & k_{\textrm{Continuous}}(x,x')
\end{align*}
\item The gradient of $\frac{\partial\mathcal{L}}{\partial\sigma_{2}}$
\begin{align*}
\frac{\partial\mathcal{L}}{\partial\sigma_{2}}= & \frac{\partial\mathcal{L}}{\partial k_{z}}\times\frac{\partial k_{z}}{\partial\sigma_{2}}\\
\frac{\partial k_{z}}{\partial\sigma_{2}}= & \left(1-\lambda\right)k_{\textrm{time2}}(t,t')\frac{\partial k_{\textrm{Categorical}}}{\partial\sigma_{2}}+\lambda k_{xt}k_{\textrm{time2}}(t,t')\frac{\partial k_{\textrm{Categorical}}}{\partial\sigma_{2}}\\
\frac{\partial k_{\textrm{Categorical}}}{\partial\sigma_{2}}= & k_{\textrm{Categorical}}(h,h')
\end{align*}
\item The gradient of $\frac{\partial\mathcal{L}}{\partial\lambda}$
\begin{align*}
\frac{\partial\mathcal{L}}{\partial\lambda}= & \frac{\partial\mathcal{L}}{\partial k_{z}}\times\frac{\partial k_{z}}{\partial\lambda}\\
\frac{\partial k_{z}}{\partial\lambda}= & -\left(k_{ht}+k_{xt}\right)+k_{ht}k_{xt}
\end{align*}
\end{itemize}


\end{document}